\documentclass{article}

 \usepackage[preprint]{neurips_2026}

\usepackage[utf8]{inputenc} %
\usepackage[T1]{fontenc}    %
\usepackage{hyperref}       %
\usepackage{url}            %
\usepackage{booktabs}       %
\usepackage{amsfonts}       %
\usepackage{amsmath}
\usepackage{amssymb}
\usepackage{amsthm}
\usepackage{nicefrac}       %
\usepackage{microtype}      %
\usepackage{xcolor}         %
\usepackage{algorithm}
\usepackage{algpseudocode}
\usepackage{bm}
\usepackage{wrapfig}
\usepackage{graphicx}       %
\usepackage{multirow}       %

\newtheorem{intheorem}{Informal Theorem}
\newtheorem{theorem}{Theorem}

\newtheorem{corollary}{Corollary}

\title{Task-Induced Riemannian Metrics for Vision Transformer Feature Spaces}

\author{%
    Andrew Bond \thanks{Corresponding author: \href{mailto:andrew.bond@epfl.ch}{andrew.bond@epfl.ch}}\\
    Koç University \\
    Istanbul, Türkiye \\
    \And
    Ege Erdem Özlü \\
    Koç University \\
    Istanbul, Türkiye \\
    \And
    Tuna Çimen \\
    Koç University \\
    Istanbul, Türkiye \\
    \And 
    Ilkin Umut Melanlioglu \\
    Koç University \\
    Istanbul, Türkiye \\
    \And
    Tolga Birdal \\
    Imperial College London \\
    London, UK \\
    \And
    Erkut Erdem \\
    Hacettepe University \\
    Ankara, Türkiye \\
    \And
      Aykut Erdem\\
    Koç University \\
    Istanbul, Türkiye \\
}

\begin{document}

\maketitle

\begin{abstract}

Methods operating on Vision Transformer (ViT) feature spaces typically rely on Euclidean distance or cosine similarity. This assumes that every direction is equally meaningful, but there is no reason to believe the true task geometry has this property. The task-sensitive geometry of the feature space is given by the pullback metric $g(F) = J(F)^\top J(F)$, where $J$ is the Jacobian of the decoder's output fed to a task-specific distance, with respect to the features. Storing the full $g$ is infeasible at modern scales, and for dense outputs such as depth maps even forming $J$ is impractical. We show that whether a low-rank approximation of this metric can be learned depends on the model-decoder pair, and we characterize this with a matrix-free diagnostic $\kappa_{cap}(r)$ computable with a low number of Jacobian-vector products. For tractable pairs, we develop the \textit{Spectral Pullback Network} (SPN), which learns a low-rank version of the metric from randomized power iteration, and we distill it into a $310$K-parameter importance head that predicts token importance directly from the features. When the Jacobian spectrum is too spread out for a low-rank approximation, passing the decoder's input features through a VAE bottleneck can restore tractability. Across DPT, DINOv2, CLIP, and VGGT backbones, $\kappa_{cap}(r)$ predicts which learned-metric architectures are viable. The importance head reaches Spearman $\rho = 0.998$ on DINOv2 CLS, and our geometric token pruning reduces the additional depth error of ToMe-based token selection by $25\%$ on DPT depth at prune ratio $0.5$, without fine-tuning the ViT. Project page: \url{https://cyberiada.github.io/TaskInducedViTs/}.\par
\end{abstract}

\section{Introduction}
Vision Transformer \citep{transformer, vit} feature representations are commonly compared using Euclidean distance or cosine similarity. This is convenient, but it assumes that all directions in feature space are equally relevant. Downstream decoders rarely behave this way. A change in a ViT feature may strongly affect a depth map, a class
embedding, or a camera-pose estimate in some directions, while changes of similar Euclidean size in other directions may have almost no effect. Thus, the similarity that matters for a task is not simply ``are these features close?'', but rather ``do these features induce similar task outputs?''

The decoder itself defines such a task-induced local geometry. Let \(F \in \mathbb{R}^{N \times D}\) be a ViT feature map at some layer and let \(\phi\) be a downstream decoder (with an optional task-specific distance) such as DPT for depth, the CLS embedding used for classification, or VGGT for camera pose. If \(J(F)\) is the Jacobian of the decoder's output with respect to \(F\), then $v^\top J(F)^\top J(F) v = \|J(F)v\|^2$ measures how much the task output changes when \(F\) is perturbed infinitesimally in direction \(v\). The pullback metric \(J(F)^\top J(F)\) therefore tells us which feature directions the decoder is sensitive to and which it nearly ignores. It is a local, first-order object, so it says nothing about two distant feature maps. The corresponding global notion is the geodesic distance under this metric, which we only approximate (Appendix~\ref{sec:task_retrieval}). \par

Storing this metric is infeasible at ViT scale, since $J^\top J$ has $\mathcal{O}(N^2 D^2) \approx 4 \times 10^{10}$ entries for a ViT-B/14 feature map. The Jacobian $J \in \mathbb{R}^{M \times ND}$ can be formed for small outputs, such as a 9-dimensional camera pose. For dense outputs such as a depth map ($M \approx 5 \times 10^4$), $J$ has $\sim 10^{10}$ entries, and only its products with vectors are practical. This raises two questions. \textit{Is the task sensitivity concentrated enough that a low-rank metric can capture it?} If so, \textit{can we learn it in a form that is useful at inference time, without needing backpropagation or iterative algorithms?} \par

We make three contributions.

\textbf{(1) A tractability diagnostic.} We introduce a scalar $\kappa_{cap}(r) \in [0, 1]$, measuring the fraction of decoder sensitivity captured by the top-$r$ singular directions of $J$. A second measurement, the coefficient of variation (CV) of the singular values, tells whether these directions are concentrated on a few tokens or spread across many. Both are computed offline with matrix-free Jacobian-vector products. We provide three theoretical results supporting our claims: a ceiling on the sensitivity that any rank-$r$ restriction of the metric retains (Theorem~\ref{inthm:rank_r}), a bound showing when per-token scoring is uninformative (Theorem~\ref{thm:impossibility}), and a convergence rate for power iteration (Theorem~\ref{thm:convergence}). The first and third adapt classical results to our setting. \par

\textbf{(2) A learned low-rank metric and its distillation.} The \textit{Spectral Pullback Network} (SPN) predicts a low-rank factorization of the task-induced metric from ViT features and is trained with supervision from randomized power iteration. For applications that only need token importance, we distill \citep{distillation} it into a small ($\sim$310K-parameter) \textit{importance head} that predicts per-token scores from features alone. When no low-rank approximation is adequate in feature space, a VAE bottleneck \citep{s2vae} on the features the decoder uses makes one possible again. \par

\textbf{(3) Geometric token pruning.} For CLS-based decoders, we replace cosine-based token merging with a task-sensitive merge rule. For dense decoders, we combine hard pruning with a standard \textit{Last-Layer Fusion} \cite{token_cropr} step that re-attaches pruned tokens before decoding. A first-order error bound guides both (Theorem~\ref{thm:pruning}). \par

Empirically, the diagnostic separates the backbone--decoder pairs we evaluate into four regimes (depending on the $\kappa_{cap}$ and CV) that predict which architecture is appropriate. On DINOv2 CLS, the importance head reaches Spearman $\rho = 0.998$ with its Jacobian-derived target. On DPT depth, geometric pruning reduces the additional error of ToMe-based token selection by 25\% at prune ratio 0.50, while ToMe remains competitive at the lowest ratios (Table~\ref{tab:dpt_pruning}). These results suggest that learning \textit{which} directions matter can be as valuable as learning better features. \par

\section{Related Work}
\paragraph{Metric learning and pullback geometry.} Classical metric learning fits a Mahalanobis metric from pairwise constraints \citep{xing_metric, lmnn} or contrastive embeddings \citep{contrastive}. Neural methods either output an SPD matrix directly \citep{rml_ot, taskmet}, sometimes only at keypoints \citep{geometry_hvae, data_aug_vae}, which needs $O(d^2)$ memory and is prohibitive at $d = ND$, or recover the metric implicitly, for example as a pullback through a learned map \citep{latent_space_oddity}. Pullback metrics are also used to study the latent spaces of generative models \citep{latent_space_oddity, geometry_gans, understanding_diffusion_models}. Our $g = J^\top J$ is the Gauss--Newton form of the Fisher metric of the decoder output \citep{natural_gradient, info_geometry}, estimated on a frozen ViT feature space rather than over model parameters or a generative latent space. The SPN represents it implicitly through its leading eigenspace.
\vspace{-4mm}
\paragraph{Token importance in ViTs.} Prior analyses use input-gradient saliency, attention aggregated across layers \citep{attention_rollout}, or attention weights and entropies \citep{attention_entropy} as proxies for token importance. We instead differentiate a probe map with respect to intermediate features and study the full pullback $J^\top J$. Large attention or gradient magnitude does not by itself separate task-sensitive from task-insensitive directions.
\vspace{-4mm}
\paragraph{Token pruning and merging.} ToMe \citep{tome} merges redundant tokens by cosine similarity with bipartite matching. DynamicViT \citep{dynamic_vit}, A-ViT \citep{avit}, and EViT \citep{evit} discard tokens using learned or attention-derived scores, and Token Cropr \citep{token_cropr} trains a scoring head end-to-end against task losses. DynamicViT, A-ViT, and Token Cropr fine-tune the backbone, whereas we keep the backbone--decoder pair frozen, so they are not direct comparisons in our setting. None of these methods measures token similarity with the decoder-induced geometry, which we do (Eq.~\ref{eq:dtask}), and we add theorems on when per-token scoring can succeed (Theorems~\ref{thm:impossibility} and~\ref{thm:pruning}).

\section{Background and Problem Formulation}
\label{sec:background}

\begin{figure}[t]
    \centering
    \includegraphics[width=0.88\linewidth]{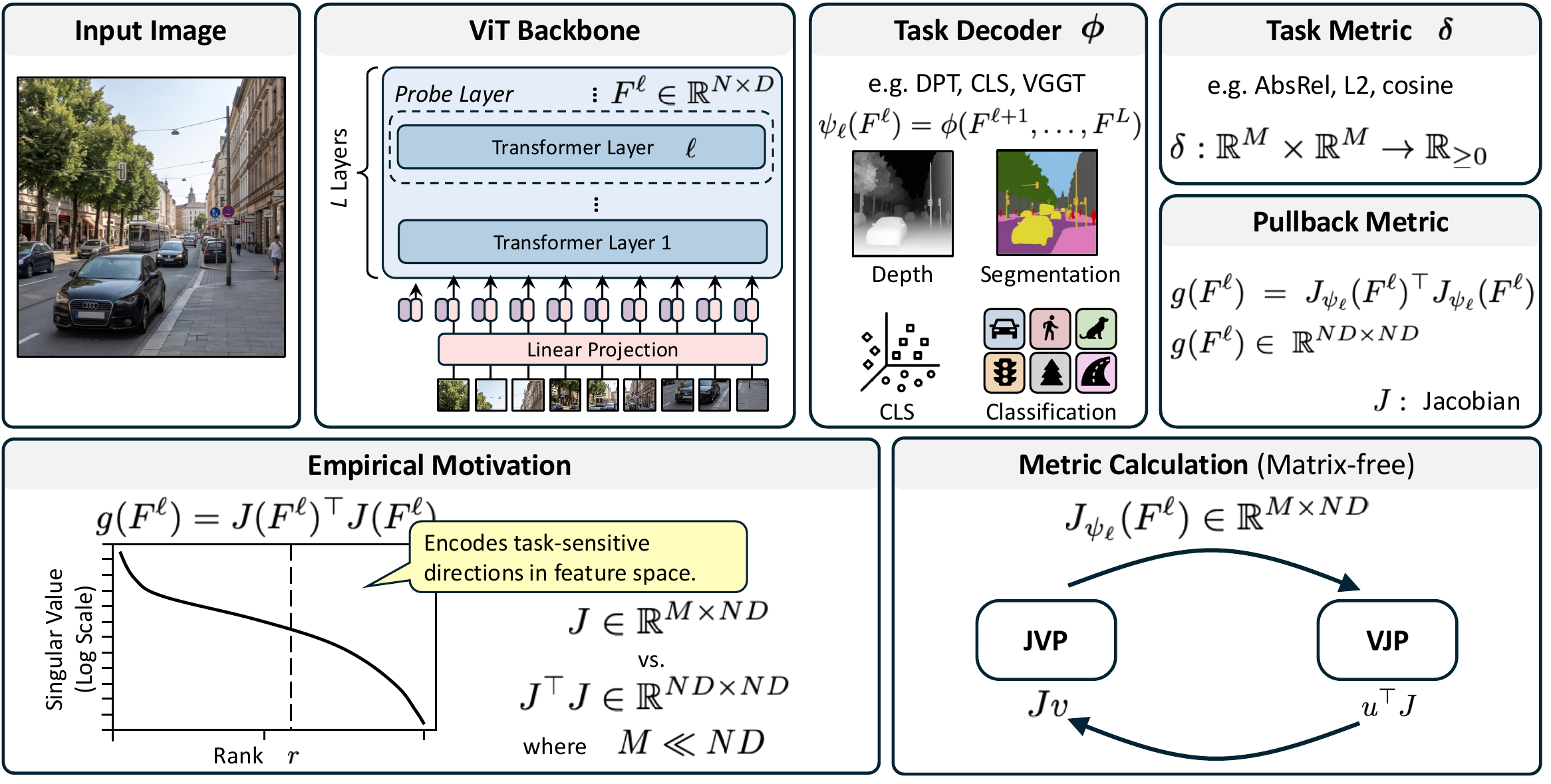}
    \caption{\textbf{The task-induced pullback metric and how we use it.} A ViT feature map $F^\ell\in\mathbb{R}^{N\times D}$ at probe layer $\ell$ passes through the remaining ViT layers and a task decoder $\phi$ (optionally followed by a task metric $\delta$), composing into the probe map $\psi_\ell : \mathbb{R}^{N \times D} \to \mathbb{R}^M$. Its Jacobian $J_{\psi_\ell} \in \mathbb{R}^{M \times N D}$ defines the pullback metric $g(F^\ell) = J^\top J \in \mathbb{R}^{N D \times N D}$, accessed only via JVP/VJP at  $\mathcal{O}(m+rq)$ cost. Rapid singular-value decay indicates that a rank-$r$ truncation can capture most of the task sensitivity.
    }
    \label{fig:overview}
\end{figure}

\paragraph{Problem formulation.} At layer $\ell$, a ViT produces features $F^\ell \in \mathbb{R}^{N \times D}$ ($N = 257$, $D = 768$ for ViT-B/14 at $224 \times 224$). A \textbf{task decoder} $\phi$ takes one or more feature layers and produces a vectorized output in $\mathbb{R}^M$, for example a DPT depth map \citep{dpt} ($M \sim 50{,}000$), the final-layer CLS embedding itself ($M = D = 768$, with no classification head), or a VGGT camera pose ($M = 9$: translation, quaternion, and field of view). For a probe layer $\ell$, the \textbf{probe map} $\psi_\ell$ composes the remaining ViT layers with the decoder,
\begin{equation}
    \psi_\ell(F^\ell) \;=\; \phi\!\left(F^{\ell+1}, \dots, F^L\right) \;\in\; \mathbb{R}^M.
\end{equation}
Its Jacobian $J_{\psi_\ell}(F^\ell) \in \mathbb{R}^{M \times ND}$ defines the \textbf{pullback metric} on the feature space,
\begin{equation}
    g(F^\ell) \;=\; J_{\psi_\ell}(F^\ell)^\top J_{\psi_\ell}(F^\ell) \;\in\; \mathbb{R}^{ND \times ND}.
\end{equation}
For a perturbation $v \in \mathbb{R}^{ND}$, $v^\top g(F^\ell) v = \|J_{\psi_\ell}\, v\|^2$ is the squared change in the decoder output, so large values mark task-sensitive directions. A task distance $\delta$ on the output (scale-invariant log error, cosine, LPIPS \cite{lpips}, etc.) can be composed with the probe map, which we suppress in the notation. Appendix~\ref{sec:riemannian_review} reviews the Riemannian background \citep{metric_learning_review}.

\paragraph{What kind of metric is $g$?} $g$ is the pullback of the Euclidean output metric along $\psi_\ell$. Since $M \ll ND$ for every decoder we consider, $g$ is positive semidefinite, i.e.\ a degenerate (singular) metric \citep{singular_riemannian}, or equivalently a feature-to-output sensitivity form on the ambient feature space. Its kernel is the set of feature perturbations the decoder ignores. We make no manifold/sub-manifold assumptions on the features and use no tangent spaces. We claim only that the top-$r$ eigendirections are the ambient directions of greatest decoder sensitivity. Our results (Theorems~\ref{thm:captured_sensitivity}--\ref{thm:pruning}) are stated for this PSD form, and the metric deployed at inference (Eq.~\ref{eq:spn}) adds an isotropic floor $\varepsilon > 0$ that makes it positive definite.

\paragraph{Choice of output metric.} The Euclidean output metric $J^\top J$ is a modeling choice. For any positive-definite output metric $H = LL^\top$, the pullback $J^\top H J = (LJ)^\top (LJ)$ is the Euclidean pullback of the decoder followed by $L$, so all our methods apply with $J$ replaced by $LJ$. To validate this, on the VGGT camera head, an SO(3)-aware $H$ leaves tractability unchanged ($\kappa_{cap} = 0.99$ under both) and the token importance nearly unchanged (Spearman $\rho = 0.995$ between the two rankings, Appendix~\ref{sec:pullback_vit}). We thus use the Euclidean metric in the rest of the paper.

\section{When is Metric Learning Tractable?}
\label{sec:tractability}
Before approximating or training anything, we want to know whether a rank-$r$ metric (for some fixed choice of $r$) can represent the task sensitivity of a given backbone--decoder pair, and whether that sensitivity is concentrated on a few tokens, or diffuse. We now introduce two cheap measurements that answer these questions, and provide the relevant theory.

\subsection{Captured Energy, Effective Rank, and Spatial Concentration}
A rank-$r$ pullback metric is most useful when its top-$r$ directions account for most of the task sensitivity. We measure this with the \emph{captured energy fraction}
\begin{equation}
    \kappa_{cap}(r) = \frac{\sum_{j=1}^r \sigma_j^2}{\|J\|_F^2} \in [0, 1],
\end{equation}
where $\sigma_j$ is the $j$-th largest singular value of $J$. By Theorem~\ref{inthm:rank_r}, $\kappa_{cap}(r)$ is the quality of the best rank-$r$ metric and an upper bound on the sensitivity that any restriction of $g$ to an $r$-dimensional subspace can retain. We also use the \textit{entropic effective rank} $r_{\text{eff}}(J) = \exp(H)$, computed from the spectral entropy $H$ of the normalized squared singular values. Both $\|J\|_F^2$ and $H$ are estimated without forming $J$, using Hutchinson's estimator \citep{hutchinson_trace} and Stochastic Lanczos Quadrature \citep{slq_algorithm}. The top-$r$ singular vectors are recovered by $q$ rounds of randomized power iteration \citep{halko_randomized, randomized_nla}. Together these give the diagnostic in Algorithm~\ref{alg:hutchinson} (Appendix~\ref{sec:algorithms}) at a cost of $\mathcal{O}(m + rq)$ JVP/VJP pairs, about $140$ at our defaults $m = 100$, $r = 20$, $q = 2$.

The recovered top-$r$ right singular vectors $v_1, \dots, v_r$ also reveal which of the tokens of the feature map carry the task-sensitive directions. Splitting each $v_j \in \mathbb{R}^{ND}$ into $N$ token blocks $v_j^{(t)} \in \mathbb{R}^D$ for $t = 1, \dots, N$, we measure how unevenly the energy of $v_j$ is distributed across tokens via the coefficient of variation
\begin{equation}
\label{eq:cv}
    \mathrm{CV} = \frac{1}{r} \sum_{j=1}^r \frac{\mathrm{std}_t\!\left(\|v_j^{(t)}\|_2\right)}{\mathrm{mean}_t\!\left(\|v_j^{(t)}\|_2\right)}.
\end{equation}
Low CV ($\approx 0$) means that the recovered singular vectors put roughly equal mass on every token. This is the \emph{delocalized} regime, where Theorem~\ref{thm:impossibility} shows that per-token scoring becomes uninformative. Higher CV means that the energy is concentrated on fewer tokens, and per-token scorers are then viable. CV therefore provides a way to predict whether a per-token architecture will work before one is trained.

\subsection{Structural Characterization of Tractability}
\label{sec:structural}

\begin{intheorem}[Captured Sensitivity Ceiling]

\label{inthm:rank_r}Consider the rank-$r$ approximation formed from the top-$r$ singular directions of $J(F)$, together with an isotropic floor $\epsilon I$. Among approximations of this form, the top-$r$ eigenspace gives the minimum Frobenius error. Moreover, as $\epsilon \to 0$, the expected sensitivity captured by this approximation is exactly $\kappa_{cap}(r)$. Thus, $\kappa_{cap}(r)$ is the largest fraction of the expected task sensitivity that can be retained by restricting $g$ to an $r$-dimensional subspace. \par
\end{intheorem}

\begin{intheorem}[Impossibility of Factored Metrics]
\label{thm:impossibility}
A \emph{factored} (per-token) metric is one that scores each token independently and combines the per-token scores additively, with no cross-token interaction. If a unit task-sensitive direction $v^*$ is $\zeta$-delocalized (each token block has squared norm in $[(1-\zeta)/N, (1+\zeta)/N]$), then in expectation over the feature distribution, every factored metric assigns nearly the same sensitivity to $v^*$ as to a uniformly random direction, up to relative error $\zeta$. The bound scales linearly with $\zeta$, so near-delocalized directions leave a factored scorer with proportionally little signal. (Proof: Appendix~\ref{sec:theorem_proofs}, Theorem~\ref{thm:impossibility_formal}.)
\end{intheorem}
\noindent\emph{Remark.} The amount of delocalization matters in practice. In \S\ref{sec:conformal_head} we train a small network, the importance head, to predict one score per token: how strongly that token's features contribute to the top-$r$ singular directions of $J$, weighted by the singular values (Eq.~\ref{eq:imp_target}). On DINOv2 CLS L10 ($\mathrm{CV} \approx 4.06$), its predictions match this target almost perfectly (Spearman $\rho = 0.998$). On delocalized DPT depth ($\mathrm{CV} \approx 0.026$ at the L02 and L10 probes of Table~\ref{tab:regime_taxonomy}), the same network, reaches only $\rho \approx 0.20$. Trained instead to predict each token's full Jacobian norm $\|J_{:,t}\|_F$, it reaches only $\rho \approx 0.09$ (Appendix~\ref{sec:conformal_head_ablations}). This is the weakness of per-token scoring that Theorem~\ref{thm:impossibility} predicts.

\begin{intheorem}[Convergence--Tractability]
\label{thm:convergence}
Run $q$ rounds of block power iteration on $J^\top J$ \citep{halko_randomized, randomized_nla}, starting from a Gaussian random sketch with $r_0 \geq r$ columns. With high probability over the sketch, the resulting subspace approximates the true rank-$r$ task-sensitive subspace with subspace-alignment defect $\leq C \cdot (\sigma_{r+1}/\sigma_r)^{2q}$, where $C$ depends on the sketch but not on $q$. This is a geometric convergence rate in $q$, not a tight finite-$q$ guarantee. The constant $C$ grows as $\mathcal{O}(ND\, r_0)$ and can be very large, so the bound alone does not tell us how many iterations suffice. In practice we increase $q$ until the recovered subspace stabilizes. Flat spectra need more iterations ($q \approx 15$ in our experiments). (Proof: Appendix~\ref{sec:theorem_proofs}, Theorem~\ref{thm:convergence_formal}.)
\end{intheorem}

\subsection{Regime Taxonomy}
\label{sec:regime_taxonomy}
We apply this analysis (written fully as Algorithm~\ref{alg:hutchinson} in the Appendix) to six backbone--decoder--probe-layer combinations spanning DPT depth \citep{depth_anything_v2}, DINOv2 \citep{dinov2}, and VGGT \citep{vggt}. An analogous CLIP CLS measurement appears in Appendix~\ref{sec:clip_cls_results}. Table~\ref{tab:regime_taxonomy} reports $\kappa_{cap}(20)$, the per-token CV, and the full-spectrum effective rank $r_{\text{eff}}^{\text{SLQ}}$. The results fall into four distinct regimes.

\begin{wraptable}[17]{r}{0.56\textwidth}
\vspace*{-20pt}
\centering
\footnotesize
\setlength{\tabcolsep}{4pt}
\caption{Tractability diagnostic ($r{=}20$, $m{=}100$, $q{=}2$, 50 images). Layer indices give the probe layer, \emph{(1-layer)} marks a single-layer Depth Anything V2 (DA V2) probe, and \emph{(full)} marks the full-resolution VGGT depth output. For the two VGGT dense rows, $\kappa_{cap}(20)$ is measured at the listed probe layer and $r_{\text{eff}}^{\text{SLQ}}$ at the input to VGGT's heads (App.~\ref{sec:multi_layer}). CV is N/A for VGGT camera (rank-9 output). The last column gives the regime of \S\ref{sec:regime_taxonomy}.}
\label{tab:regime_taxonomy}
\begin{tabular}{lcccc}

\toprule
Config & $\kappa_{cap}(20)$ & CV & $r_{\text{eff}}^{\text{SLQ}}$ & Regime \\
\midrule
DA V2 depth, L02            & 95.6\%   & 0.026 & 34    & (ii)  \\
DA V2 depth, L10 (1-layer)  & 96.6\% & 0.026 & 20   & (ii)  \\
DINOv2 CLS, L10             & 33.4\%   & 4.06  & 270   & (iii) \\
VGGT camera, L16            & $\approx 100\%$ & N/A & 9 & (iv)  \\
VGGT pointcloud, L08        & 96.6\%   & 0.59  & 212   & (i)   \\
VGGT depth, L16 (full)      & 76.2\%   & 0.63  & 282   & (i)   \\
\bottomrule
\end{tabular}
\end{wraptable}

\paragraph{Four regimes.} The six rows fall into four cases, listed in the last column of Table~\ref{tab:regime_taxonomy}, and each calls for a different architecture.

\emph{(i) Rich spectrum.} The two dense VGGT heads, depth and pointcloud, have effective ranks that exceed any practical rank budget by an order of magnitude, so no rank-$r$ metric captures most of the task sensitivity. The VAE pathway of \S\ref{sec:vae_extension} addresses this regime and brings the effective rank down to $[30, 53]$ for VGGT depth.

\emph{(ii) Low rank, delocalized.} The two Depth Anything V2 cases, which use the same DPT decoder at layers 2 and 10, fit within the rank budget, but their task-sensitive directions are spread evenly over tokens. By Theorem~\ref{thm:impossibility} a per-token metric cannot represent them, so learning the full metric needs cross-token attention (the SPN, \S\ref{sec:spn}). Per-token scores remain usable for pruning but are weak ($\rho \approx 0.20$), which is why dense pruning also relies on Last-Layer Fusion and multi-stage schedules (\S\ref{sec:token_pruning}).

\emph{(iii) Marginal $\kappa$, high CV.} The DINOv2 CLS Jacobian has a rich spectrum, so a rank-$20$ metric captures only a third of the task sensitivity. The regime is still workable because the concentration is spatial rather than spectral: a few tokens carry the task-sensitive directions, so per-token scoring is viable even though a faithful low-rank metric is out of reach. The importance head reaches $\rho = 0.998$ here, and CLIP CLS behaves similarly (Appendix~\ref{sec:clip_cls_results}).

\emph{(iv) Structurally rank-bounded.} The VGGT camera head outputs $M = 9$ scalars per view, so $\mathrm{rank}(J) \leq 9 < r$ and $\kappa_{cap}(20) \approx 100\%$ follows from the rank bound rather than from spectral concentration. This makes it a controlled setting for evaluating the SPN.

Two further observations hold across the rows. The convergence rate of Theorem~\ref{thm:convergence} holds, and two power-iteration steps suffice where the spectral gap is large, while flatter spectra need more (Appendix~\ref{sec:theorem_proofs}). A decoder that uses several feature layers is not intractable in itself, and stays tractable when those layers are close together (Appendix~\ref{sec:multi_layer}).

\section{Learning the Task-Induced Metric}
\label{sec:our_algorithm}
When the diagnostic indicates that a rank-$r$ metric is a meaningful target, we want a network that predicts it from features alone, so that the decoder's Jacobian is not needed at inference (as this would first require producing the output and applying backprop, losing all possible gains). We first see why simple supervision does not work, then introduce the SPN and its training procedure, and describe the small importance head used for token pruning.

\subsection{Why Naive Supervision Fails}
\label{sec:naive_fails}
When trying to train a model to predict the metrics from features alone, there are several natural ideas. \emph{Distance matching}, which fits $\|f_\theta(F_1) - f_\theta(F_2)\|$ to $\delta(\phi(F_1), \phi(F_2))$, constrains only the values and not the derivatives, and thus leaves the local geometry underdetermined between sampled pairs. Even a perfect match would approximate a straight-line distance, not geodesic distance. \emph{Random-direction regression}, which fits $\|J v\|^2$ on unit vectors $v$, fails because a uniform $v \in \mathbb{R}^{ND}$ overlaps the rank-$r$ subspace with probability $r/ND \approx 10^{-5}$, so the signal is dominated by directions the metric does not model, and requiring an infeasible amount of training steps to learn anything. Theorem~\ref{thm:convergence} motivates power iteration instead. Each step on $J^\top J$ amplifies the alignment with the task-sensitive subspace, so $q$ steps will make the JVP/VJP results into a usable supervision signal.

\subsection{The Spectral Pullback Network (SPN)}
\label{sec:spn}
We parametrize the rank-$r$ pullback metric as a neural network $g_\theta$ that takes probe-layer features $F$ and returns a symmetric positive-definite matrix in spectral (eigen-decomposition) form:
\begin{equation}
\label{eq:spn}
    v^\top g_\theta(F) v \;=\; \sum_{j=1}^{r} \underbrace{\lambda_j(F; \theta)}_{\text{learned}}\, \big(\underbrace{u_j(F; \theta)}_{\text{learned}}{}^\top v\big)^{2} \;+\; \underbrace{\varepsilon}_{\text{fixed}} \|v\|^{2}, \qquad U(F;\theta)^\top U(F;\theta) = I_r.
\end{equation}
The $r$ mode vectors $u_j(F; \theta) \in \mathbb{R}^{ND}$ are the columns of $U \in \mathbb{R}^{ND \times r}$, which is constrained to be orthonormal. They are the network's estimates of the top-$r$ right singular vectors $v_j$ of $J(F)$, and the energies $\lambda_j(F; \theta) > 0$ predict the squared singular values $\sigma_j^2$. Both depend on the input features $F$ and the trainable parameters $\theta$. The constant $\varepsilon > 0$ is a fixed regularizer that keeps $g_\theta$ positive definite at inference, and is required for the Woodbury product below. Because this floor is isotropic, it does not change the informative rank-$r$ eigenspace or the importance scores. The captured-energy and Eckart--Young statements of Theorem~\ref{inthm:rank_r} therefore concern the structured part (the $\varepsilon \to 0$ limit), not the full-rank version. When $u_j \to v_j$ and $\lambda_j \to \sigma_j^2$, the SPN recovers the rank-$r$-plus-isotropic metric of Theorem~\ref{inthm:rank_r}.
\vspace{-3mm}
\paragraph{Architecture.} The network maps $F$ through a linear projection ($D \to H$), two pre-norm multi-head attention blocks, and a per-token projection of $Dr$ coordinates, which are reshaped to $ND \times r$ and QR-orthonormalized into $U(F;\theta)$. The energy network is a global pool, an MLP, and a Softplus, and it produces the $r$ positive energies $\lambda(F;\theta)$. At inference, the product $g_\theta(F)\, v = \varepsilon v + U(\lambda \odot (U^\top v))$ costs $O(NDr)$ rather than $O((ND)^2)$. Per-token importance is available in closed form, $\mathrm{imp}_\theta(t) = \sqrt{\sum_j \lambda_j \,\|u_j^{(t)}\|^2}$, where $u_j^{(t)} \in \mathbb{R}^{D}$ denotes the $t$-th token block of $u_j$.

Cross-token attention is necessary in the delocalized regime. When the singular vectors are spread evenly across tokens (low CV), computing $u_j^\top v = \sum_t u_j^{(t)\top} v^{(t)}$ requires summing contributions from all $N$ tokens, which per-token architectures cannot do. Theorem~\ref{thm:impossibility} formalizes this failure mode.

\subsection{RandNLA Training}
\label{sec:randnla_training}
The SPN obtains supervision from power iteration using Jacobian-vector and vector-Jacobian products, without explicitly forming $J$. At each training step we sample a Gaussian matrix $\Omega \in \mathbb{R}^{ND \times r_0}$ with $r_0 \geq r$. We apply $q$ rounds of power iteration $\Omega \gets J^\top(J \Omega)$, using alternating JVP and VJP calls and re-orthonormalizing between rounds \citep{saad_iterative}, and take the first $r$ columns as the target mode matrix $\widehat V_r$. By Theorem~\ref{thm:convergence}, $\widehat V_r$ converges exponentially to the true top-$r$ task-sensitive subspace as $q$ grows. The training loss has two parts. The first is a Grassmannian \emph{subspace loss} $\mathcal{L}_{\text{sub}} = 1 - \tfrac{1}{r}\|U(F;\theta)^\top \widehat V_r\|_F^2$, which measures the angle between the predicted and target subspaces and does not depend on the order of the modes. The second is an \emph{energy regression} $\mathcal{L}_{\text{en}} = \mathrm{MSE}(\log \lambda(F;\theta),\, \log \widehat S^2)$ on the empirical squared norms $\widehat S_j^2 = \|J \hat v_j\|^2$ (Algorithm~\ref{alg:randnla}, Appendix~\ref{sec:algorithms}). When two adjacent singular values are nearly degenerate ($\sigma_j / \sigma_{j+1} < 1.05$), the modes can swap between images, so we exclude such modes from $\mathcal{L}_{\text{en}}$ for that image. The subspace loss is unaffected by this, though. If the SPN fits these targets exactly, its sensitivity ratio converges to the ceiling $\kappa_{cap}(r)$ as $q$ grows (Corollary~\ref{cor:ssm_quality}). In practice a single network shared across images does not fit every image exactly (Appendix~\ref{sec:spn_vae_validation}).
\subsection{Importance Head and Inference-Time Distillation}
\label{sec:conformal_head}
Applications that only need a per-token importance score, such as token pruning (\S\ref{sec:token_pruning}), do not need the SPN at inference time. The \textbf{importance head} $h_\theta : \mathbb{R}^{N \times D} \to \mathbb{R}^N$ is a small network ($\sim$310K parameters) that predicts one scalar per token directly from $F$, with no Jacobian access and no SPN call at inference. Its target is the \textbf{exact (Jacobian-derived) importance}
\begin{equation}
\label{eq:imp_target}
    \mathrm{imp}^*(F)_t \;=\; \sqrt{\sum_{j=1}^{r} \sigma_j^{2}\, \big\|v_j^{(t)}\big\|_2^{2}},
\end{equation}
to which the SPN's $\mathrm{imp}_\theta(t) = \sqrt{\sum_j \lambda_j \|u_j^{(t)}\|^{2}}$ converges as $u_j \to v_j$ and $\lambda_j \to \sigma_j^2$. Since $\mathrm{imp}^*(F)$ depends only on the top-$r$ pairs $\{(\sigma_j, v_j)\}$, which a per-image randomized SVD recovers (Algorithm~\ref{alg:randnla}), we train $h_\theta$ from cached targets without instantiating the SPN (although the SPN can also be used). The $\sigma^2$-weighted top-$r$ projection keeps spatial signal even when the plain block norm $\|J_{:,t}\|_F$ becomes uniform across tokens (Theorem~\ref{thm:impossibility}, with the ablation in Appendix~\ref{sec:conformal_head_ablations}).
\vspace{-3mm}
\paragraph{Architecture and training.} The head is one cross-token MHA block (4 heads), a per-token MLP, and a Softplus output, with loss
\begin{equation}
\label{eq:conformal_loss}
    \mathcal{L}(\theta) = \mathbb{E}_F\!\left[\big\|\log h_\theta(F) - \log \mathrm{imp}^*(F)\big\|^{2}\right] + w_{\text{rank}}\, \mathcal{L}_{\text{rank}}\!\left(h_\theta(F),\, \mathrm{imp}^*(F)\right), \quad w_{\text{rank}} = 0.5.
\end{equation}
The ranking term reflects that pruning depends on token ordering, not absolute scale. Theorem~\ref{thm:pruning} motivates the target: with uniform residuals, the best pruning set removes the tokens with the smallest Jacobian block norms (Corollary~\ref{cor:opt_prune}). $\mathrm{imp}^*$ keeps the top-$r$ part of these block norms for the probe-map Jacobian at the prune layer, which we use as a computable proxy for the decoder Jacobian at the final layer that appears in the bound.

\subsection{Empirical Validation}
\label{sec:empirical_validation}
On DINOv2 CLS L10, the importance head trained on 2{,}000 ImageNet-val \citep{imagenet} images reaches Spearman $\rho = 0.998 \pm 0.001$ against the target $\mathrm{imp}^*$. Here $\rho$ is the rank correlation over the tokens of each image, averaged over a 200-image held-out set. The head costs a small fraction of a backbone forward pass. The $\sigma^2$-weighted target is essential, as replacing it with plain block norms lowers $\rho$ to $\approx 0.09$ on DPT, predicted by Theorem~\ref{thm:impossibility} (ablation in Appendix~\ref{sec:conformal_head_ablations}). The SPN itself reaches the deterministic supervision ceiling at the cheapest training setting on the VAE-compressed VGGT depth pipeline (\S\ref{sec:vae_extension}, Appendix~\ref{sec:spn_vae_validation}).

\paragraph{Why not learn token sensitivity directly?} A simpler alternative is to directly train on per-token sensitivity, i.e.\ the plain Jacobian block norms $\|J_{:,t}\|_F$, without the spectral construction. This is the ablation above. On DPT depth it reaches $\rho \approx 0.09$, against $\rho \approx 0.20$ for the $\sigma^2$-weighted top-$r$ target, which can only be computed through the spectral construction. Where sensitivity is concentrated on a few tokens (high CV), a directly trained head can be enough. Where it is delocalized, the spectral target gives the head a signal it can learn. The SPN also provides more than a score per token, because its rank-$r$ factors define a task-induced quadratic form on whole feature maps, and has many other potential applications.

\section{Extending to Intractable Settings via VAE Compression}
\label{sec:vae_extension}
Now we consider the case where the diagnostic rules out a low-rank metric because the task sensitivity is spread over far more directions than any practical rank. Dense decoders for depth, point clouds, or image reconstruction produce outputs in $\mathbb{R}^M$ with $M$ in the tens of thousands. When the resulting Jacobian's effective rank exceeds the rank budget by an order of magnitude, no rank-$r$ SPN can be a faithful approximation in the original feature space. This is the case for VGGT depth and pointcloud, whose effective ranks at the input to VGGT's heads are $282$ and $212$ (regime~(i) of \S\ref{sec:regime_taxonomy}). 

Note that the rank of a given Jacobian must be no larger than the minimum of the dimension of the input and output spaces, so $rk(J) \leq \min(M, ND)$. Since $M$ is a property of the output space and cannot be changed, the logical way to reduce the Jacobian rank is to reduce $ND$. We do this with a pretrained variational autoencoder (VAE) on the ViT features which the task decoders take as input. For VGGT, let $x_{\text{vis}}$ be the concatenated patch tokens of the four layers read by the heads (layers 4, 11, 17, 23). The VAE encoder $E$ maps $x_{\text{vis}}$ to a latent $z \in \mathbb{R}^{M'}$ made of 32 register tokens of dimension 128 ($M' = 4096$). Its decoder $G$ reconstructs the features, and the frozen task head takes as input the reconstruction,
\begin{equation}
    \tilde\psi(x_{\text{vis}}) \;=\; \phi\big(G(E(x_{\text{vis}}))\big) \;\in\; \mathbb{R}^{M}.
\end{equation}
Every path from the features to the output passes through $z$, so $\mathrm{rank}(J_{\tilde\psi}) \leq M'$ however large $M$ is. Empirically, the compression moves the effective spectrum into the tractable regime. VGGT depth probed at $x_{\text{vis}}$ has $r_{\text{eff}}^{\text{SLQ}} = 282$ (regime~i). With the VAE bottleneck in place, $r_{\text{eff}}^{\text{SLQ}} \in [30, 53]$ and $r_{0.90} \approx 40$ across our test configurations, a $\approx 6\times$ reduction that places the compressed metric in regime~(ii). The same compression applies to VGGT pointcloud ($r_{\text{eff}}^{\text{SLQ}} = 212$ at $x_{\text{vis}}$). To train an SPN we probe directly at the latent. The map $z \mapsto \phi(G(z))$ has a Jacobian with only $M'$ columns, so its right singular vectors live in $\mathbb{R}^{M'}$ and exact SVD targets are cheap to compute. In all our experiments we use the $S^2$-VAE \citep{s2vae} with power-spherical bottlenecks \citep{power_spherical}, a VAE tailored specifically to ViT features, which reaches a depth-reconstruction quality of $\delta_1 \approx 0.99$ on VGGT depth.
\vspace{-3mm}
\paragraph{Approximation induced by VAE compression.} The compressed pullback $g_{\tilde\psi}$ is the pullback of the task metric through the VAE-reconstructed features, not through the original ones, so the metric we learn here is not literally $g$ but an approximation of it. When the VAE reconstructs the task-relevant content of the features well, the two pullbacks are expected to agree approximately in the task-sensitive directions.
\vspace{-3mm}
\paragraph{Experimental validation.} On the VAE-compressed VGGT depth pipeline, probing at the latent $z$, RandNLA with a single power-iteration step ($q = 1$) reaches a subspace alignment of $V_1 = 0.713$ with the true rank-$r$ task-sensitive subspace. This essentially matches the deterministic \emph{supervision ceiling} of $V_1 = 0.710$, which is the best this network reaches when given exact SVD targets at every step. A randomly initialized network has $V_1 = 0.005$, so the gap between the random and trained models reflects learned task geometry. The saturation at the cheapest setting, $q = 1$, is consistent with Theorem~\ref{thm:convergence}, which predicts fast convergence when the spectral gap is large. It supports both the VAE pathway and the RandNLA algorithm. Appendix~\ref{sec:spn_vae_validation} gives the full setup and ablations.

\section{Token Merging and Pruning via Geometric Importance}
\vspace{-2mm}
\label{sec:token_pruning}
We now use the importance head to remove or merge tokens in a frozen ViT while keeping the decoder's output close to the unpruned one. The prune ratio $\eta$ is the fraction of patch tokens removed or merged. Let $F^L$ and $\hat F$ be the final-layer features without and with pruning, let $\mathcal{T}$ be the set of removed tokens, and let $\rho_t$ be the part of token $t$'s feature update that is lost after correction (Appendix~\ref{sec:pruning_details}). With $J_{:,t}(\hat F)$ the block of the decoder Jacobian for token $t$ and $\beta$ a bound on the decoder's Hessian, a Taylor expansion (Theorem~\ref{thm:pruning}) gives
\begin{equation}
\label{eq:pruning_bound}
    \big\|\phi(F^L) - \phi(\hat F)\big\|_2 \;\leq\; \sum_{t \in \mathcal{T}} \underbrace{\big\|J_{:,t}(\hat F)\big\|_F}_{\text{token importance}}\, \underbrace{\|\rho_t\|_2}_{\text{residual}} \;+\; \tfrac{\beta}{2} \sum_{t \in \mathcal{T}} \|\rho_t\|_2^{2}.
\end{equation}
The first term is small when the removed tokens have low task sensitivity, which is what the importance head targets. The second term depends on how much of the removed tokens' content is recovered.

\subsection{CLS Decoders: Soft Merge}
\label{sec:soft_merge}
A CLS decoder uses only the CLS token and is invariant to permutations of the other tokens, so tokens can be merged instead of removed. We follow ToMe \citep{tome}. The $\eta N$ tokens with the lowest importance scores are each merged into their nearest remaining token, with importance-weighted averaging. We replace ToMe's cosine similarity with the rank-$r$ pullback metric restricted to a single token (Eq.~\ref{eq:dtask}, Appendix~\ref{sec:pruning_details}). In our experiments this distance uses the SPN prediction for $J$. All CLS strategies share this matching step and differ only in which tokens they merge.
\subsection{Dense Decoders: Hard Prune + Last-Layer Fusion}
\label{sec:hard_prune}
Dense decoders such as DPT \citep{dpt} reassemble the tokens into a 2D map, which merging corrupts. We instead hard-prune the lowest-importance tokens at layer $\ell$, freeze their features, and re-insert them at their original positions wherever the decoder needs. The frozen features miss the updates of the later layers. We use \textbf{Last-Layer Fusion (LLF)} (ex. \cite{token_cropr}) to reduce this drift. It re-inserts the pruned tokens before the final ViT block and runs that block on the full sequence, so they attend to the updated kept tokens. LLF adds no parameters. All dense strategies (Random, ToMe score, and Importance) share this pipeline and differ only in which tokens they remove. We report the additional scale-invariant log error (SILog, $\times 100$) relative to the unpruned prediction.

\subsection{Results}
\label{sec:pruning_results}

\begin{table}[t]
\centering
\caption{Token pruning vs.\ baselines, ImageNet-val @ 224 (mean $\pm$ std, 3 seeds; lower is better). Bold marks the best method per column. Importance outperforms ToMe at every ratio for CLS and at $\eta \geq 0.2$ for DPT.}
\label{tab:cls_pruning}
\renewcommand{\arraystretch}{0.88}
\resizebox{0.82\textwidth}{!}{
\begin{tabular}{lccccccc}
\toprule
Method & $\eta{=}0.05$ & $0.10$ & $0.15$ & $0.20$ & $0.30$ & $0.40$ & $0.50$ \\
\midrule
\multicolumn{8}{l}{\emph{CLS cosine-distance degradation ($\times 100$), DINOv2-B/14}}\\
Random     & 0.17\,\scriptsize$\pm$0.03 & 0.36\,\scriptsize$\pm$0.05 & 0.60\,\scriptsize$\pm$0.08 & 0.86\,\scriptsize$\pm$0.09 & 1.49\,\scriptsize$\pm$0.14 & 2.31\,\scriptsize$\pm$0.22 & 3.44\,\scriptsize$\pm$0.31 \\
ToMe score & 0.67\,\scriptsize$\pm$0.11 & 0.78\,\scriptsize$\pm$0.07 & 0.95\,\scriptsize$\pm$0.07 & 1.12\,\scriptsize$\pm$0.11 & 1.75\,\scriptsize$\pm$0.13 & 2.56\,\scriptsize$\pm$0.25 & 3.82\,\scriptsize$\pm$0.42 \\
Importance  & \textbf{0.001}\,\scriptsize$\pm$0.000 & \textbf{0.007}\,\scriptsize$\pm$0.001 & \textbf{0.024}\,\scriptsize$\pm$0.002 & \textbf{0.060}\,\scriptsize$\pm$0.005 & \textbf{0.25}\,\scriptsize$\pm$0.03 & \textbf{0.80}\,\scriptsize$\pm$0.09 & \textbf{2.09}\,\scriptsize$\pm$0.13 \\
\midrule
\multicolumn{8}{l}{\emph{DPT additional SILog ($\times 100$), Depth-Anything V2 single-stage [4]}}\\
Random     & 4.59\,\scriptsize$\pm$0.09 & 3.70\,\scriptsize$\pm$0.10 & \textbf{3.31}\,\scriptsize$\pm$0.14 & \textbf{3.19}\,\scriptsize$\pm$0.19 & 3.16\,\scriptsize$\pm$0.17 & 3.20\,\scriptsize$\pm$0.15 & 3.29\,\scriptsize$\pm$0.18 \\
ToMe score & \textbf{4.29}\,\scriptsize$\pm$0.07 & \textbf{3.66}\,\scriptsize$\pm$0.09 & 3.51\,\scriptsize$\pm$0.17 & 3.50\,\scriptsize$\pm$0.10 & 3.52\,\scriptsize$\pm$0.13 & 3.67\,\scriptsize$\pm$0.13 & 4.09\,\scriptsize$\pm$0.13 \\
Importance  & 4.88\,\scriptsize$\pm$0.06 & 4.86\,\scriptsize$\pm$0.07 & 3.85\,\scriptsize$\pm$0.11 & 3.28\,\scriptsize$\pm$0.15 & \textbf{3.02}\,\scriptsize$\pm$0.16 & \textbf{3.01}\,\scriptsize$\pm$0.16 & \textbf{3.08}\,\scriptsize$\pm$0.15 \\
\bottomrule
\end{tabular}}
\vspace{-3mm}
\label{tab:dpt_pruning}
\end{table}

Table~\ref{tab:cls_pruning} shows the two settings. On DINOv2 CLS, importance-guided merging has the lowest degradation at every ratio, and at low ratios it is more than an order of magnitude below ToMe scoring. On DPT depth, importance reduces the additional error of ToMe scoring by $25\%$ at $\eta = 0.50$ but loses at the lowest ratios. The CLS result carries over to CLIP and to ImageNet-Sketch, and importance also beats frozen-backbone baselines such as EViT (Appendices~\ref{sec:clip_cls_results} and~\ref{sec:extra_baselines}). ImageNet top-1 accuracy stays between $0.66$ and $0.68$ for every method and ratio ($0.674$ unpruned). For dense depth, the 2-stage [4, 8] schedule beats ToMe scoring at every ratio and resolution on NYU-Depth-V2, by $25$--$35\%$, while single-stage results depend on the dataset. Pruning removes $12.1\%$ of the backbone's theoretical FLOPs at $\eta = 0.20$, but a wall-clock speedup appears only at $448 \times 448$ ($1.12\times$). Appendix~\ref{sec:pruning_details} discusses when importance-guided pruning wins, a non-monotone trend in Table~\ref{tab:dpt_pruning}, and cost.

\section{Conclusion}
\label{sec:conclusion}
\vspace{-3mm}

We studied the geometry a task decoder induces on a frozen ViT feature space, given by the pullback metric $g = J^\top J$. Since g cannot be stored, and for dense outputs J cannot even be formed, we introduced a matrix-free diagnostic, $\kappa_{cap}(r)$ that measures how much of the decoder's sensitivity a rank-r metric can capture, together with a measure of how that sensitivity spreads across tokens. Across DPT, DINOv2, CLIP and VGGT probes, these two numbers sort backbone–decoder pairs into four regimes that predict which architecture is viable, and a VAE bottleneck can be used to move the intractable pairs into a tractable regime. Where the diagnostic allows it, the Spectral Pullback Network learns the metric from randomized power iteration, and a $310$K-parameter importance head distills it into per-token scores that reach $\rho = 0.998$ on DINOv2 CLS. Using those scores for token pruning cuts the added depth error of ToMe-based selection by 25\% at prune ratio 0.5, with the backbone and decoder left frozen. Appendices M and O give limitations and future directions.

\section*{Acknowledgements}
\vspace{-2mm}
This research was supported in part by the KUIS AI Center Research Awards to Andrew Bond. Aykut Erdem acknowledges funding from the TÜBİTAK 2247-A Program under Award 123C550. TB acknowledges support from the UKRI Engineering and Physical Sciences Research Council (EPSRC) through the Future Leaders Fellowship [grant number MR/Y018818/1]. 

\bibliography{neurips_2026}

\begin{thebibliography}{10}

\bibitem{transformer}
Ashish Vaswani, Noam Shazeer, Niki Parmar, Jakob Uszkoreit, Llion Jones, Aidan~N Gomez, {\L}ukasz Kaiser, and Illia Polosukhin.
\newblock Attention is all you need.
\newblock In {\em Advances in Neural Information Processing Systems}, 2017.

\bibitem{vit}
Alexey Dosovitskiy, Lucas Beyer, Alexander Kolesnikov, Dirk Weissenborn, Xiaohua Zhai, Thomas Unterthiner, Mostafa Dehghani, Matthias Minderer, Georg Heigold, Sylvain Gelly, Jakob Uszkoreit, and Neil Houlsby.
\newblock An image is worth 16x16 words: Transformers for image recognition at scale.
\newblock In {\em International Conference on Learning Representations}, 2021.

\bibitem{distillation}
Geoffrey Hinton, Oriol Vinyals, and Jeff Dean.
\newblock Distilling the knowledge in a neural network.
\newblock {\em arXiv preprint arXiv:1503.02531}, 2015.

\bibitem{s2vae}
Andrew Bond, Ilkin~Umut Melanlioglu, Erkut Erdem, and Aykut Erdem.
\newblock Beyond gaussian bottlenecks: Topologically aligned encoding of vision-transformer feature spaces, 2026.

\bibitem{token_cropr}
Benjamin Bergner, Christoph Lippert, and Aravindh Mahendran.
\newblock Token cropr: Faster {ViTs} for quite a few tasks.
\newblock In {\em Proceedings of the IEEE/CVF Conference on Computer Vision and Pattern Recognition}, 2025.

\bibitem{xing_metric}
Eric~P Xing, Andrew~Y Ng, Michael~I Jordan, and Stuart Russell.
\newblock Distance metric learning, with application to clustering with side-information.
\newblock In {\em Advances in Neural Information Processing Systems}, 2002.

\bibitem{lmnn}
Kilian~Q Weinberger and Lawrence~K Saul.
\newblock Distance metric learning for large margin nearest neighbor classification.
\newblock {\em Journal of Machine Learning Research}, 10:207--244, 2009.

\bibitem{contrastive}
Raia Hadsell, Sumit Chopra, and Yann LeCun.
\newblock Dimensionality reduction by learning an invariant mapping.
\newblock In {\em IEEE Conference on Computer Vision and Pattern Recognition}, 2006.

\bibitem{rml_ot}
Christopher Scarvelis and Justin Solomon.
\newblock Riemannian metric learning via optimal transport.
\newblock In {\em The Eleventh International Conference on Learning Representations}, 2023.

\bibitem{taskmet}
Dishank Bansal, Ricky~TQ Chen, Mustafa Mukadam, and Brandon Amos.
\newblock Taskmet: Task-driven metric learning for model learning.
\newblock {\em Advances in Neural Information Processing Systems}, 36:46505--46519, 2023.

\bibitem{geometry_hvae}
Cl{\'e}ment Chadebec, Cl{\'e}ment Mantoux, and St{\'e}phanie Allassonni{\`e}re.
\newblock Geometry-aware hamiltonian variational auto-encoder.
\newblock {\em arXiv preprint arXiv:2010.11518}, 2020.

\bibitem{data_aug_vae}
Cl{\'e}ment Chadebec, Elina Thibeau-Sutre, Ninon Burgos, and St{\'e}phanie Allassonni{\`e}re.
\newblock Data augmentation in high dimensional low sample size setting using a geometry-based variational autoencoder.
\newblock {\em IEEE Transactions on Pattern Analysis and Machine Intelligence}, 45(3):2879--2896, 2022.

\bibitem{latent_space_oddity}
Georgios Arvanitidis, Lars~Kai Hansen, and S{\o}ren Hauberg.
\newblock Latent space oddity: On the curvature of deep generative models.
\newblock In {\em 6th International Conference on Learning Representations, ICLR 2018}, 2018.

\bibitem{geometry_gans}
Binxu Wang and Carlos~R Ponce.
\newblock A geometric analysis of deep generative image models and its applications.
\newblock In {\em International Conference on Learning Representations}, 2021.

\bibitem{understanding_diffusion_models}
Yong-Hyun Park, Mingi Kwon, Jaewoong Choi, Junghyo Jo, and Youngjung Uh.
\newblock Understanding the latent space of diffusion models through the lens of riemannian geometry.
\newblock {\em Advances in Neural Information Processing Systems}, 36:24129--24142, 2023.

\bibitem{natural_gradient}
Shun-ichi Amari.
\newblock Natural gradient works efficiently in learning.
\newblock {\em Neural Computation}, 10(2):251--276, 1998.

\bibitem{info_geometry}
Shun-ichi Amari and Hiroshi Nagaoka.
\newblock {\em Methods of Information Geometry}.
\newblock American Mathematical Society, 2000.

\bibitem{attention_rollout}
Samira Abnar and Willem Zuidema.
\newblock Quantifying attention flow in transformers.
\newblock In {\em Proceedings of the 58th Annual Meeting of the Association for Computational Linguistics}, pages 4190--4197, 2020.

\bibitem{attention_entropy}
Jesse Vig and Yonatan Belinkov.
\newblock Analyzing the structure of attention in a transformer language model.
\newblock In {\em Proceedings of the 2019 ACL Workshop BlackboxNLP}, pages 63--76, 2019.

\bibitem{tome}
Daniel Bolya, Cheng-Yang Fu, Xiaoliang Dai, Peizhao Zhang, Christoph Feichtenhofer, and Judy Hoffman.
\newblock Token merging: Your {ViT} but faster.
\newblock In {\em The Eleventh International Conference on Learning Representations}, 2023.

\bibitem{dynamic_vit}
Yongming Rao, Wenliang Zhao, Benlin Liu, Jiwen Lu, Jie Zhou, and Cho-Jui Hsieh.
\newblock {DynamicViT}: Efficient vision transformers with dynamic token sparsification.
\newblock In {\em Advances in Neural Information Processing Systems}, 2021.

\bibitem{avit}
Hongxu Yin, Arash Vahdat, Jose~M Alvarez, Arun Mallya, Jan Kautz, and Pavlo Molchanov.
\newblock {A-ViT}: Adaptive tokens for efficient vision transformer.
\newblock In {\em Proceedings of the IEEE/CVF Conference on Computer Vision and Pattern Recognition}, 2022.

\bibitem{evit}
Youwei Liang, Chongjian Ge, Zhan Tong, Yibing Song, Jue Wang, and Pengtao Xie.
\newblock Not all patches are what you need: Expediting vision transformers via token reorganizations.
\newblock In {\em International Conference on Learning Representations}, 2022.

\bibitem{dpt}
Ren{\'e} Ranftl, Alexey Bochkovskiy, and Vladlen Koltun.
\newblock Vision transformers for dense prediction.
\newblock In {\em 2021 IEEE/CVF International Conference on Computer Vision (ICCV)}, pages 12159--12168. IEEE, 2021.

\bibitem{lpips}
Richard Zhang, Phillip Isola, Alexei~A Efros, Eli Shechtman, and Oliver Wang.
\newblock The unreasonable effectiveness of deep features as a perceptual metric.
\newblock In {\em 2018 IEEE/CVF conference on computer vision and pattern recognition}, pages 586--595. IEEE, 2018.

\bibitem{metric_learning_review}
Samuel Gruffaz and Josua Sassen.
\newblock A review on riemannian metric learning: Closer to you than you imagine.
\newblock {\em arXiv preprint arXiv:2503.05321}, 2025.

\bibitem{singular_riemannian}
Alessandro Benfenati and Alessio Marta.
\newblock A singular riemannian geometry approach to deep neural networks i. theoretical foundations.
\newblock {\em Neural Networks}, 158:331--343, 2023.

\bibitem{hutchinson_trace}
Michael~F Hutchinson.
\newblock A stochastic estimator of the trace of the influence matrix for laplacian smoothing splines.
\newblock {\em Communications in Statistics-Simulation and Computation}, 18(3):1059--1076, 1989.

\bibitem{slq_algorithm}
Shashanka Ubaru, Jie Chen, and Yousef Saad.
\newblock Fast estimation of tr(f(a)) via stochastic lanczos quadrature.
\newblock {\em SIAM Journal on Matrix Analysis and Applications}, 38(4):1075--1099, 2017.

\bibitem{halko_randomized}
Nathan Halko, Per-Gunnar Martinsson, and Joel~A Tropp.
\newblock Finding structure with randomness: Probabilistic algorithms for constructing approximate matrix decompositions.
\newblock {\em SIAM Review}, 53(2):217--288, 2011.

\bibitem{randomized_nla}
Per-Gunnar Martinsson and Joel~A Tropp.
\newblock Randomized numerical linear algebra: Foundations and algorithms.
\newblock {\em Acta Numerica}, 29:403--572, 2020.

\bibitem{depth_anything_v2}
Lihe Yang, Bingyi Kang, Zilong Huang, Zhen Zhao, Xiaogang Xu, Jiashi Feng, and Hengshuang Zhao.
\newblock Depth anything {V2}.
\newblock In {\em Advances in Neural Information Processing Systems}, 2024.

\bibitem{dinov2}
Maxime Oquab, Timoth{\'e}e Darcet, Th{\'e}o Moutakanni, et~al.
\newblock {DINOv2}: Learning robust visual features without supervision.
\newblock {\em Transactions on Machine Learning Research}, 2024.

\bibitem{vggt}
Jianyuan Wang, Minghao Chen, Nikita Karaev, Andrea Vedaldi, Christian Rupprecht, and David Novotny.
\newblock {VGGT}: Visual geometry grounded transformer.
\newblock In {\em Proceedings of the IEEE/CVF Conference on Computer Vision and Pattern Recognition}, 2025.

\bibitem{saad_iterative}
Yousef Saad.
\newblock {\em Iterative Methods for Sparse Linear Systems}.
\newblock SIAM, 2 edition, 2003.

\bibitem{imagenet}
Jia Deng, Wei Dong, Richard Socher, Li-Jia Li, Kai Li, and Li~Fei-Fei.
\newblock {ImageNet}: A large-scale hierarchical image database.
\newblock In {\em IEEE Conference on Computer Vision and Pattern Recognition}, pages 248--255, 2009.

\bibitem{power_spherical}
Nicola De~Cao and Wilker Aziz.
\newblock The power spherical distribution.
\newblock In {\em Proceedings of the 37th International Conference on Machine Learning, INNF+ Workshop}, 2020.

\bibitem{eckart_young}
Carl Eckart and Gale Young.
\newblock The approximation of one matrix by another of lower rank.
\newblock {\em Psychometrika}, 1(3):211--218, 1936.

\bibitem{mirsky}
Leon Mirsky.
\newblock Symmetric gauge functions and unitarily invariant norms.
\newblock {\em The Quarterly Journal of Mathematics}, 11(1):50--59, 1960.

\bibitem{vershynin}
Roman Vershynin.
\newblock {\em High-Dimensional Probability: An Introduction with Applications in Data Science}.
\newblock Cambridge University Press, 2018.

\bibitem{davis_kahan}
Chandler Davis and William~M Kahan.
\newblock The rotation of eigenvectors by a perturbation. {III}.
\newblock {\em SIAM Journal on Numerical Analysis}, 7(1):1--46, 1970.

\bibitem{stewart_sun}
G.~W. Stewart and Ji-Guang Sun.
\newblock {\em Matrix Perturbation Theory}.
\newblock Academic Press, 1990.

\bibitem{imagenet_sketch}
Haohan Wang, Songwei Ge, Zachary Lipton, and Eric~P Xing.
\newblock Learning robust global representations by penalizing local predictive power.
\newblock In {\em Advances in Neural Information Processing Systems}, 2019.

\bibitem{nyu_depth}
Nathan Silberman, Derek Hoiem, Pushmeet Kohli, and Rob Fergus.
\newblock Indoor segmentation and support inference from {RGBD} images.
\newblock In {\em European Conference on Computer Vision}, pages 746--760, 2012.

\bibitem{clip}
Alec Radford, Jong~Wook Kim, Chris Hallacy, et~al.
\newblock Learning transferable visual models from natural language supervision.
\newblock In {\em International Conference on Machine Learning}, pages 8748--8763, 2021.

\bibitem{flow_matching}
Yaron Lipman, Ricky T.~Q. Chen, Heli Ben-Hamu, Maximilian Nickel, and Matthew Le.
\newblock Flow matching for generative modeling.
\newblock In {\em International Conference on Learning Representations}, 2023.

\bibitem{riemannian_flow_matching}
Ricky T.~Q. Chen and Yaron Lipman.
\newblock Flow matching on general geometries.
\newblock In {\em International Conference on Learning Representations}, 2024.

\bibitem{integrated_gradients}
Mukund Sundararajan, Ankur Taly, and Qiqi Yan.
\newblock Axiomatic attribution for deep networks.
\newblock In {\em International Conference on Machine Learning}, pages 3319--3328, 2017.

\bibitem{smoothgrad}
Daniel Smilkov, Nikhil Thorat, Been Kim, Fernanda Vi{\'e}gas, and Martin Wattenberg.
\newblock {SmoothGrad}: removing noise by adding noise.
\newblock {\em arXiv preprint arXiv:1706.03825}, 2017.

\bibitem{energy_ood}
Weitang Liu, Xiaoyun Wang, John~D Owens, and Yixuan Li.
\newblock Energy-based out-of-distribution detection.
\newblock In {\em Advances in Neural Information Processing Systems}, 2020.

\bibitem{mahalanobis_ood}
Kimin Lee, Kibok Lee, Honglak Lee, and Jinwoo Shin.
\newblock A simple unified framework for detecting out-of-distribution samples and adversarial attacks.
\newblock In {\em Advances in Neural Information Processing Systems}, 2018.

\bibitem{sinkhorn}
Marco Cuturi.
\newblock Sinkhorn distances: Lightspeed computation of optimal transport.
\newblock In {\em Advances in Neural Information Processing Systems}, 2013.

\end{thebibliography}
\bibliographystyle{unsrt}

\newpage
\appendix

\section{Riemannian Geometry}
\label{sec:riemannian_review}
We collect the Riemannian-geometry background used in the main text, in enough depth that a reader with basic linear algebra and calculus can follow the proofs in Appendix~\ref{sec:theorem_proofs}. A more thorough treatment can be found in \cite{metric_learning_review} or any standard text on differential geometry. We state the definitions formally, but each is followed by an \emph{Intuition} note that reduces it to the linear algebra it amounts to in our setting, for those who only want the general idea and do not need the full formalism. We assume multivariable calculus and linear algebra, but no differential geometry or point-set topology. 

\subsection{Smooth Manifolds and Tangent Spaces}
A \textbf{smooth manifold} $\mathcal{M}$ of dimension $d$ is a Hausdorff topological space equipped with an atlas of charts $\{(U_\alpha, \varphi_\alpha)\}$, where each $U_\alpha \subseteq \mathcal{M}$ is open and each $\varphi_\alpha : U_\alpha \to \mathbb{R}^d$ is a homeomorphism, such that the transition maps $\varphi_\beta \circ \varphi_\alpha^{-1}$ are smooth wherever defined. Familiar examples are $\mathbb{R}^n$ itself, the $n$-sphere $S^n$, and matrix groups such as $\mathrm{SO}(n)$. In our setting, the manifold of interest is the ViT feature space at probe layer $\ell$, taken as the open set $\mathbb{R}^{N \times D}$ with the global chart $\varphi(F^\ell) = \mathrm{vec}(F^\ell)$.

\emph{Intuition.} The definition says that $\mathcal{M}$ can be covered by coordinate systems, each identifying a neighbourhood with an open subset of $\mathbb{R}^d$, and that the change-of-coordinate maps are smooth, so differentiation does not depend on which coordinates are used. This means that a region around a given point looks like $\mathbb{R}^d$, and we can smoothly switch between two regions while preserving smoothness. Our $\mathcal{M}$ is the vector space $\mathbb{R}^{N \times D}$ with the single global chart $\mathrm{vec}(\cdot)$, which is the degenerate case: one coordinate system, no transition maps, and every construction below reduces to linear algebra on the vectorized feature tensor. The manifold language buys us one thing, namely a metric that is allowed to vary from point to point.

At each point $p \in \mathcal{M}$, the \textbf{tangent space} $T_p\mathcal{M}$ is a $d$-dimensional real vector space whose elements are derivations of smooth real-valued functions at $p$. Concretely, an element $v \in T_p\mathcal{M}$ is the velocity $\dot\gamma(0)$ of some smooth curve $\gamma : (-\varepsilon, \varepsilon) \to \mathcal{M}$ with $\gamma(0) = p$. Two curves give the same tangent vector iff they have the same first-order Taylor expansion in any chart. The disjoint union $T\mathcal{M} = \bigsqcup_p T_p\mathcal{M}$ is itself a smooth manifold, the \textbf{tangent bundle}.

\emph{Intuition.} $T_p\mathcal{M}$ is the space of first-order displacements at $p$, i.e.\ the directions along which a function can be differentiated there. For a vector space it is canonically identified with the space itself, so a tangent vector at $F^\ell$ is a perturbation $\Delta F \in \mathbb{R}^{N \times D}$ of the same shape, and the curve $\gamma(t) = F^\ell + t\,\Delta F$ gives it as a velocity. The definition through derivations exists because a general manifold has no ambient space in which to take differences of nearby points.

A smooth map $\phi : \mathcal{M} \to \mathcal{N}$ between manifolds induces a \textbf{differential} (or pushforward) at each $p$, $d\phi_p : T_p\mathcal{M} \to T_{\phi(p)}\mathcal{N}$, defined by $d\phi_p(v) = \dot{(\phi \circ \gamma)}(0)$ for any curve $\gamma$ with $\dot\gamma(0) = v$. In coordinates this is the Jacobian matrix.

\emph{Intuition.} The differential is the linear map that best approximates $\phi$ near $p$, in the sense that $\phi(p + \epsilon v) = \phi(p) + \epsilon\, d\phi_p(v) + O(\epsilon^2)$. In coordinates it is the Jacobian, and applying it to a single direction is a Jacobian-vector product, the primitive our algorithms are built on.

\subsection{Riemannian Metrics, Lengths, and Geodesics}
A \textbf{Riemannian metric} $g$ on $\mathcal{M}$ assigns to each $p$ an inner product $g_p : T_p\mathcal{M} \times T_p\mathcal{M} \to \mathbb{R}$ that varies smoothly with $p$. In a chart $(U, \varphi)$ with coordinates $x^1, \dots, x^d$, $g$ is represented by a smoothly varying symmetric positive-definite (SPD) matrix $g_{ij}(p)$, with $g_p(u, v) = \sum_{ij} g_{ij}(p)\, u^i v^j$. We use the same letter $g$ for the metric and its coordinate matrix, $g(F) \in \mathbb{R}^{ND \times ND}$ in our case.

\emph{Intuition.} A Riemannian metric is a field of quadratic forms. At each point it is a Mahalanobis-style inner product with an SPD matrix $g_{ij}(p)$, so the squared length of a perturbation $v$ is $v^\top g(p)\, v$. It departs from the Euclidean case in two ways: the matrix need not be the identity, so directions have unequal cost and can be correlated, and it depends on $p$, so the same perturbation costs different amounts at different points. For $g = J^\top J$ the eigenvectors are the right singular vectors of $J$ and the eigenvalues are $\sigma_j^2$, so length is measured in units of output change.

The metric defines an infinitesimal notion of length on $\mathcal{M}$: a smooth curve $\gamma : [0, 1] \to \mathcal{M}$ has \textbf{length}
\begin{equation}
    L(\gamma) \;=\; \int_0^1 \sqrt{g_{\gamma(t)}\big(\dot\gamma(t), \dot\gamma(t)\big)}\, dt.
\end{equation}
The \textbf{Riemannian distance} $d_g(p, q)$ between two points is then the infimum of $L(\gamma)$ over piecewise-smooth curves with $\gamma(0) = p$ and $\gamma(1) = q$. This $d_g$ satisfies all metric-space axioms whenever $\mathcal{M}$ is connected and $g$ is positive definite. If $g$ is only positive semidefinite, as for the pullback metrics below, $d_g$ is a pseudometric, meaning that distinct points can be at distance zero.

\emph{Intuition.} The length integrates the local norm $\|\dot\gamma(t)\|_{g(\gamma(t))}$ along the curve, and the distance is the infimum over connecting curves. If $g$ were constant, say $g = L^\top L$, the integral would collapse to $d_g(p, q) = \|L(p - q)\|_2$ and the minimizers would be straight lines, which is the familiar Mahalanobis distance. The challenge of the Riemannian case is that $g$ varies with position, so the minimizer generally bends away from the straight line towards regions where the metric is cheap.

A \textbf{geodesic} is a critical point of $L$ in the calculus-of-variations sense. Equivalently, it satisfies the Euler--Lagrange (geodesic) equation
\begin{equation}
    \ddot\gamma^{\,k} + \Gamma^k_{ij}(\gamma)\, \dot\gamma^{\,i} \dot\gamma^{\,j} \;=\; 0,
    \qquad
    \Gamma^k_{ij}(p) \;=\; \tfrac{1}{2}\, g^{k\ell}(p)\!\left(\partial_i g_{j\ell} + \partial_j g_{i\ell} - \partial_\ell g_{ij}\right)\!(p),
\end{equation}
where $g^{k\ell}$ is the inverse metric tensor and $\Gamma^k_{ij}$ are the \textbf{Christoffel symbols}. Geodesics generalize straight lines. In flat $\mathbb{R}^n$ with the Euclidean metric, $\Gamma^k_{ij} \equiv 0$ and the geodesic equation reduces to $\ddot\gamma = 0$.

\emph{Intuition.} Geodesics are the stationary points of the length functional, so the geodesic equation is the Euler--Lagrange condition of that variational problem. It is a second-order ODE in the coordinates of the curve, and the Christoffel symbols are built from first derivatives of $g$, which is why they vanish when $g$ is constant and the solutions are then straight lines. Integrating this ODE is what makes exact geodesic distance impractical at $ND \approx 2 \times 10^5$, and it is why Appendix~\ref{sec:task_retrieval} uses a first-order surrogate.

The \textbf{exponential map} at $p$, $\exp_p : T_p\mathcal{M} \to \mathcal{M}$, sends a tangent vector $v$ to $\gamma_v(1)$, where $\gamma_v$ is the unique geodesic with $\gamma_v(0) = p$ and $\dot\gamma_v(0) = v$. On a sufficiently small neighbourhood, $\exp_p$ is a diffeomorphism, and its inverse $\log_p$ provides \textbf{normal coordinates} centred at $p$. In these coordinates the metric looks Euclidean to first order, with $g_{ij}(p) = \delta_{ij}$ and $\partial_k g_{ij}(p) = 0$.

\emph{Intuition.} The exponential map integrates the geodesic ODE for one unit of time from initial position $p$ and initial velocity $v$, and $\log_p$ inverts it near $p$. In the resulting coordinates $g(p) = I$ and all first derivatives of $g$ vanish at $p$, so the metric agrees with the Euclidean one to first order and any deviation is second order. This is the geometric counterpart of a first-order Taylor expansion.

\subsection{Pullback Metrics}
The construction underlying our entire framework is the pullback metric. Given a smooth map $\phi : \mathcal{M} \to (\mathcal{N}, h)$ from a manifold without a specified metric to a Riemannian manifold, the \textbf{pullback metric} $\phi^* h$ on $\mathcal{M}$ is defined pointwise by
\begin{equation}
    (\phi^* h)_p(u, v) \;=\; h_{\phi(p)}\!\left(d\phi_p(u),\, d\phi_p(v)\right), \qquad u, v \in T_p\mathcal{M}.
\end{equation}
When $\phi$ is an immersion (i.e.\ $d\phi_p$ has full column rank everywhere), $\phi^* h$ is positive definite and therefore defines a Riemannian metric on $\mathcal{M}$. When $d\phi_p$ has a nontrivial kernel, $\phi^* h$ is only a positive-semidefinite tensor, and its degenerate directions are exactly those that the map collapses.

\paragraph{Singular metrics.} Most properties we wish to use (local sensitivity, geodesic distance, etc.) still hold for singular Riemannian metrics, with the caveat that moving in the direction of the kernel at a given point doesn't change anything about the output. For more information about singular Riemannian geometry in the context of deep learning, see \cite{singular_riemannian}.

\emph{Intuition.} The pullback is the quadratic form obtained by pushing a perturbation through the map and measuring it in the output, $v \mapsto h(d\phi_p(v), d\phi_p(v))$. It is a first-order object: it reports the squared output change caused by an infinitesimal input perturbation. Directions in $\ker d\phi_p$ produce no output change and receive zero length, which is why the pullback is only positive semidefinite unless $d\phi_p$ is injective.

In the Euclidean case, $\mathcal{N} = \mathbb{R}^M$ with the standard Euclidean metric. Writing $\phi$ in coordinates and using $J = J_\phi(p) \in \mathbb{R}^{M \times d}$ for its Jacobian, the pullback metric is
\begin{equation}
    (\phi^* h)_{ij}(p) \;=\; \sum_a \partial_i \phi^a(p)\, \partial_j \phi^a(p) \;=\; \big[J^\top J\big]_{ij}.
\end{equation}
A perturbation direction $v$ gets length $\sqrt{v^\top J^\top J v} = \|Jv\|_2$ under the pullback. Directions $v \in \ker J$ are collapsed to zero by $\phi$ and have zero pullback length, and directions aligned with the top right singular vectors of $J$ have the largest pullback length.

\emph{Intuition.} With a Euclidean output metric the pullback is $J^\top J$, the Gram matrix of the Jacobian, and the length of $v$ is $\|Jv\|_2$. Writing $J = U \Sigma V^\top$, the eigenvectors of $J^\top J$ are the columns of $V$ and its eigenvalues are $\sigma_j^2$, so $v^\top J^\top J v = \sum_j \sigma_j^2 \langle v_j, v \rangle^2$. The rest of the paper follows from this decomposition: the top-$r$ terms carry most of the output response, the tail contributes little, and $\kappa_{cap}(r)$ is the fraction they account for.

\subsection{Specialization to ViT Probe Maps}
\label{sec:pullback_vit}
In our setting, $\mathcal{M} = \mathbb{R}^{N \times D}$ is the feature space at probe layer $\ell$, $\mathcal{N} = \mathbb{R}^M$ is the task-output space, and the smooth map is the probe map $\psi_\ell : \mathcal{M} \to \mathcal{N}$, which composes the remaining ViT layers, the task decoder, and optionally the task metric. The pullback of the Euclidean metric on $\mathbb{R}^M$ through $\psi_\ell$ is the ViT pullback metric of the main paper, $g(F) = J_{\psi_\ell}(F)^\top J_{\psi_\ell}(F)$. Because $M$ may be much smaller than $ND$ (e.g.\ $M = 9$ for VGGT camera pose, against $ND \approx 2\times 10^{5}$), $J_{\psi_\ell}$ typically has a nontrivial kernel, and $g(F)$ is positive semidefinite rather than strictly positive. The directions in $\ker J_{\psi_\ell}$ are the feature-space perturbations that leave the task output unchanged to first order, i.e.\ the directions the decoder ignores. We work on the ambient space $\mathbb{R}^{N \times D}$ with its global chart and make no assumption that the features lie on a lower-dimensional submanifold.

\paragraph{Why a metric and not just a Jacobian.} In principle, all of our diagnostics could be phrased directly in terms of $J$. We work with the pullback $g = J^\top J$ instead because (i)~$g$ lives intrinsically on the feature space and is invariant to post-composition of $\phi$ with isometries of $\mathbb{R}^M$, (ii)~its eigendecomposition gives canonical directions and energies, and (iii)~it makes the connection to classical Riemannian metric learning explicit, allowing direct comparison with explicit-SPD methods such as those of \cite{rml_ot, taskmet}. The role of the rank-$r$ approximation $g_r$ is then to truncate $g$ to its most informative directions while keeping it a (positive-semidefinite) metric tensor.

\paragraph{Choice of output metric.} We have currently formulated everything in terms of the Euclidean output metric $J^\top J$, rather than a generic pullback $J^\top H J$. However, this is not required for our method. For any positive-definite output metric $H = LL^\top$, the $H$-weighted pullback is $J^\top H J = (LJ)^\top (LJ)$, which is the Euclidean pullback of the decoder post-composed with the fixed linear map $L$. The diagnostic, SPN, and importance head therefore apply unchanged with $J$ replaced by $LJ$, at the cost of one extra matrix-vector product per JVP or VJP. For heterogeneous outputs such as camera pose, the principled choice is a block $H$ that respects the geometry of rotations and puts translation and rotation on comparable scales. On the VGGT camera head we compared the Euclidean metric with such an SO(3)-aware, block-scaled $H$. This $H$ uses the linearized geodesic metric on rotations, removes the non-physical quaternion-norm direction, and rescales translation and rotation to comparable units. Under this new version, tractability remains unchanged ($\kappa_{cap} = 0.99$ under both), per-token importance rankings correlate at Spearman $\rho = 0.995$, and the $50\%$-prune token sets agree on $97\%$ of tokens. We thus use the Euclidean metric in the rest of the paper, and leave exploration of $H$ to future work.

\section{Tractability Algorithms}
\label{sec:algorithms}
We expand on the two estimators used in Algorithm~\ref{alg:hutchinson}: the Hutchinson trace estimator and Stochastic Lanczos Quadrature (SLQ).

\subsection{Hutchinson Trace Estimator}
For symmetric $A \in \mathbb{R}^{n \times n}$ accessed only via matrix-vector products, the Hutchinson estimator \cite{hutchinson_trace} is
\[
    \hat T_m = \tfrac{1}{m} \sum_{i=1}^m z_i^\top A z_i, \qquad z_i \stackrel{\text{iid}}{\sim} \mathcal{D},\ \mathbb{E}[z_i] = 0,\ \mathbb{E}[z_i z_i^\top] = I.
\]
It has been shown that $\mathbb{E}[\hat T_m] = \mathrm{Tr}(A)$ and $\mathrm{Var}(\hat T_m) = \tfrac{1}{m}(\mathbb{E}[z^\top A z]^2 - \mathrm{Tr}(A)^2)$. For Rademacher probes ($z_i \in \{\pm 1\}^n$ iid), the variance is $\tfrac{2}{m}(\|A\|_F^2 - \sum_i A_{ii}^2)$. For SPD matrices (which is our focus), Chernoff-type bounds give exponential concentration: $\Pr[|\hat T_m - \mathrm{Tr}(A)| > \varepsilon\,\mathrm{Tr}(A)] \leq 2 \exp(-c m \varepsilon^2)$ for a constant $c$ depending on the conditioning. We use $m = 100$ Rademacher probes throughout, giving relative error $\lesssim 5\%$ in our experiments.

\subsection{Stochastic Lanczos Quadrature (SLQ)}
SLQ \citep{slq_algorithm} computes $\mathrm{Tr}\,f(A)$ for analytic $f$ via the identity $\mathrm{Tr}\,f(A) = \sum_i f(\lambda_i) = \sum_i e_i^\top f(A) e_i$, replacing the standard basis with random probes and approximating $f(A)v$ by Lanczos-tridiagonalisation. The convergence rate is exponential in the number of Lanczos steps $k$ and independent of the matrix dimension. We use $k = 30$ Lanczos steps with $m = 100$ probes for the full-spectrum effective rank $r_{\text{eff}}^{\text{SLQ}}$.

\subsection{Numerical Verification}
At $m = 100$, the per-image standard deviation of $\widehat{\kappa_{cap}(20)}$ is $\leq 0.02$ across 50 ImageNet-val images for all configurations in Table~\ref{tab:regime_taxonomy} (Appendix~\ref{sec:full_sweep_results} gives the full sweep). The truncated effective rank $r_{\text{eff}}^{\text{trunc}}$ stabilizes within $1\%$ at $r = 20$ and $q = 2$ for all proceed-decision configurations.

\subsection{Algorithm Pseudocode}
\label{sec:algorithm_pseudocode}

\begin{algorithm}[h]
\caption{Hutchinson Tractability Diagnostic (referenced in \S\ref{sec:tractability})}
\label{alg:hutchinson}
\begin{algorithmic}[1]
\Require Backbone, decoder $\phi$, probe layer $\ell$, rank $r$, Hutchinson samples $m$, power-iteration steps $q$
\Ensure $\kappa_{cap}(r)$, $r_{\text{eff}}^{\text{trunc}}$, $\mathrm{CV}$, decision
\State Build probe map $\psi_\ell$ by composing blocks $\ell{+}1,\dots,L$ with $\phi$ (and optionally $\delta$)
\State Wrap in \textsc{JacobianOperator} $\mathcal{J}$ (JVP/VJP interface only)
\State $S, V \gets \textsc{RandomizedSVD}(\mathcal{J}, r, q)$ \Comment{$O(rq)$ JVP+VJP pairs}
\State $\widehat{\|J\|_F^2} \gets \tfrac{1}{m}\sum_{i=1}^{m} z_i^\top \mathcal{J}^\top \mathcal{J} z_i,\ z_i \sim \mathrm{Rademacher}^{ND}$ \Comment{$O(m)$ JVP+VJP pairs}
\State $\kappa_{cap}(r) \gets \sum_{j} S_j^2 / \widehat{\|J\|_F^2}$
\State $r_{\text{eff}}^{\text{trunc}} \gets \exp\!\left(-\sum_{j} \tilde p_j \log \tilde p_j\right),\ \tilde p_j = S_j^2 / \sum_k S_k^2$
\State $\mathrm{CV} \gets \tfrac{1}{r}\sum_{j} \mathrm{std}_t(\|V_j[t]\|_2) / \mathrm{mean}_t(\|V_j[t]\|_2)$
\Return $\kappa_{cap}(r)$
\end{algorithmic}
\end{algorithm}

\begin{algorithm}[h]
\caption{RandNLA Training Step (referenced in \S\ref{sec:randnla_training})}
\label{alg:randnla}
\begin{algorithmic}[1]
\Require Probe map $\psi_\ell$, SPN $g_\theta$, rank $r$, oversample $r_0 \geq r$, power-iteration steps $q$, features $F$, energy weight $\alpha$
\State Draw $\Omega \in \mathbb{R}^{ND \times r_0}$ i.i.d.\ $\mathcal{N}(0,1)$
\For{$i = 1,\dots,q$}
    \State $\Omega \gets J^\top(J\, \Omega)$ \Comment{one VJP then one JVP per column; never form $J$}
    \State $\Omega \gets \mathrm{QR}(\Omega)$
\EndFor
\State $\widehat{V}_r \gets$ leading $r$ columns of $\Omega$ \Comment{aligned with $\mathrm{span}(V_r^*)$}
\State $\widehat{S}_j^2 \gets \|J\,\hat v_j\|^2$ for $j = 1,\dots,r$ \Comment{natural energy estimates}
\State $\mathcal{L}_{\text{sub}} \gets 1 - \tfrac{1}{r}\|U(F;\theta)^\top \widehat V_r\|_F^2$ \Comment{Grassmannian subspace loss}
\State $\mathcal{L}_{\text{en}} \gets \mathrm{MSE}\!\big(\log \lambda(F;\theta),\, \log \widehat{S}^2\big)$
\State Update $\theta$ on $\mathcal{L}_{\text{sub}} + \alpha\, \mathcal{L}_{\text{en}}$
\end{algorithmic}
\end{algorithm}

\section{Proofs of Theorems}
\setcounter{theorem}{0}
\label{sec:theorem_proofs}
This appendix collects formal statements and proofs of the four results referenced in the main text: the Captured Sensitivity Ceiling, Convergence of Power Iteration, the Token Pruning Error Bound, and Impossibility of Factored Metrics.

\subsection{Theorem 1: Captured Sensitivity Ceiling}

\begin{theorem}[Captured Sensitivity and Rank-$r$ Approximations]
\label{thm:captured_sensitivity}\label{prop:tractability}
    Let $J = U_J \Sigma V^\top$ with $\sigma_1 \geq \dots \geq \sigma_p \geq 0$, and define $g_r(F) = \sum_{j=1}^r \sigma_j^2 v_j v_j^\top + \varepsilon I$ for $\varepsilon \geq 0$. Among all approximations of the form $\tilde g = \tilde U \tilde \Lambda \tilde U^\top + \varepsilon I$ with $\tilde U \in \mathbb{R}^{ND \times r}$ orthonormal and $\tilde\Lambda \succ 0$ diagonal, $g_r$ minimises $\|g(F) - \tilde g\|_F$ in the limit $\varepsilon \to 0$ (for fixed $\varepsilon > 0$ the minimizing energies are $\sigma_j^2 - \varepsilon$). When $\sigma_r > \sigma_{r+1}$, the minimizing subspace $\mathrm{span}(V_r^*)$ is unique. Furthermore, for $v \sim \mathrm{Uniform}(\mathbb{S}^{ND-1})$ and $\varepsilon \to 0$,
    \begin{equation}
    \label{eq:sensitivity_ratio}
        \frac{\mathbb{E}_v[v^\top g_r(F) v]}{\mathbb{E}_v[v^\top g(F) v]} \;=\; \kappa_{cap}(r) \;:=\; \frac{\sum_{j=1}^r \sigma_j^2}{\|J\|_F^2}.
    \end{equation}
    No compression $P g(F) P$ of $g$ onto an $r$-dimensional subspace ($P$ a rank-$r$ orthogonal projector) achieves a larger ratio, so $\kappa_{cap}(r)$ is both the quality of $g_r$ and the ceiling for any such rank-$r$ restriction. Both quantities defining $\kappa_{cap}(r)$ can be computed to any accuracy using $\mathcal{O}(m + rq)$ JVP/VJP pairs.
\end{theorem}
\begin{proof}
    \emph{Frobenius optimality.} $g(F) = J^\top J$ is PSD with eigenvalues $\sigma_j^2$ and eigenvectors $v_j$. The Eckart--Young--Mirsky theorem \citep{eckart_young, mirsky} gives that the minimizing PSD rank-$r$ approximation (ignoring $\varepsilon I$) is $\sum_{j=1}^r \sigma_j^2 v_j v_j^\top = g_r$. Uniqueness of the subspace follows from $\sigma_r^2 > \sigma_{r+1}^2$.

    \emph{Sensitivity ratio.} For any symmetric $A$ and $v \sim \mathrm{Uniform}(\mathbb{S}^{ND-1})$, isotropy gives $\mathbb{E}_v[v^\top A v] = \mathrm{Tr}(A)/ND$. Therefore $\mathbb{E}_v[v^\top g(F) v] = \|J\|_F^2 / ND$ and $\mathbb{E}_v[v^\top g_r(F) v] = (\sum_{j=1}^r \sigma_j^2)/ND + \varepsilon$. Dividing and taking $\varepsilon \to 0$ gives Equation~\ref{eq:sensitivity_ratio}. By Ky Fan's maximum principle, $\mathrm{Tr}(P g(F) P) \leq \sum_{j=1}^r \sigma_j^2$ for every rank-$r$ orthogonal projector $P$, with equality for the projector onto $\mathrm{span}(V_r^*)$, so no such restriction achieves a higher ratio. (Without this restriction the ratio is unbounded, since scaling any $\tilde g$ scales its trace.)

    \emph{Computability.} The denominator $\|J\|_F^2 = \mathrm{Tr}(J^\top J)$ is estimated by Hutchinson's estimator with $m$ Rademacher probes, $\hat T = \tfrac{1}{m}\sum_i z_i^\top (J^\top J) z_i$, with sub-Gaussian concentration error $\mathcal{O}(1/\sqrt{m})$ \citep{hutchinson_trace}. The numerator $\sum_{j=1}^r \sigma_j^2$ is obtained from a rank-$r$ randomized SVD with $q$ power-iteration steps \citep{halko_randomized}. Each round costs $r$ JVPs and $r$ VJPs, for a total of $\mathcal{O}(rq)$ JVP/VJP pairs. The tail beyond rank $r$ contributes the remaining fraction $1 - \kappa_{cap}(r)$ of the trace.
\end{proof}

\subsection{Theorem 2: Impossibility of Factored Metrics}

\begin{theorem}[Impossibility of Factored Metrics]
\label{thm:impossibility_formal}
A metric $g_\mathcal{F}$ is \emph{factored} if $v^\top g_\mathcal{F}(F) v = \sum_{t=1}^N \varphi_t(F) \|v_t\|^2$ for some measurable $\varphi_t : \mathbb{R}^{N \times D} \to \mathbb{R}_{\geq 0}$. A unit vector $v^* \in \mathbb{R}^{ND}$ is \emph{$\zeta$-delocalized} if $\|v_t^*\|^2 \in [(1-\zeta)/N, (1+\zeta)/N]$ for all $t$. For any distribution $P$ on $\mathbb{R}^{N \times D}$ with $0 < \mathbb{E}_P[\varphi_t(F)] < \infty$, any $\zeta$-delocalized $v^*$, and $v_{\text{rand}} \sim \mathrm{Uniform}(\mathbb{S}^{ND-1})$ independent of $F \sim P$,
\begin{equation}
\label{eq:impossibility_formal}
    \left|\frac{\mathbb{E}_F[v^{*\top} g_\mathcal{F}(F) v^*]}{\mathbb{E}_F[v_{\text{rand}}^\top g_\mathcal{F}(F) v_{\text{rand}}]} - 1\right| \;\leq\; \zeta.
\end{equation}
The bound is tight. It also holds when $v^* = v^*(F)$ depends on $F$ (as the singular vectors of $J(F)$ do), provided $v^*(F)$ is $\zeta$-delocalized for every $F$.
\end{theorem}
\begin{proof}[Proof sketch]
We compute numerator and denominator separately. Tonelli's theorem (applied to non-negative measurable functions) lets us interchange expectations: $\mathbb{E}[v_{\text{rand}}^\top g_\mathcal{F}(F) v_{\text{rand}}] = \mathbb{E}_F \mathbb{E}_{v_{\text{rand}}}[v_{\text{rand}}^\top g_\mathcal{F}(F) v_{\text{rand}}]$. By isotropy of the uniform sphere, $\mathbb{E}_{v_{\text{rand}}}[\|v_{\text{rand},t}\|^2] = D/(ND) = 1/N$. Therefore
\[
    \mathbb{E}_F[v_{\text{rand}}^\top g_\mathcal{F}(F) v_{\text{rand}}] = \sum_t \tfrac{1}{N} \mathbb{E}_F[\varphi_t] =: \mu \in (0, \infty).
\]
For the numerator, since $v^*$ is fixed (non-random), $\mathbb{E}_F[v^{*\top} g_\mathcal{F}(F) v^*] = \sum_t \|v_t^*\|^2 \mathbb{E}_F[\varphi_t]$. Writing $\delta_t := N\|v_t^*\|^2 - 1 \in [-\zeta, \zeta]$ with $\sum_t \delta_t = 0$ (since $\|v^*\| = 1$), we get $\mathbb{E}_F[v^{*\top} g_\mathcal{F}(F) v^*] = \mu + \tfrac{1}{N}\sum_t \delta_t \mathbb{E}_F[\varphi_t]$. Bounding $|\tfrac{1}{N}\sum_t \delta_t \mathbb{E}_F[\varphi_t]| \leq \tfrac{\zeta}{N} \sum_t \mathbb{E}_F[\varphi_t] = \zeta \mu$ and dividing by $\mu$ gives Equation~\ref{eq:impossibility_formal}. Tightness: $N = 2$, $\|v_1^*\|^2 = (1+\zeta)/2$, $\|v_2^*\|^2 = (1-\zeta)/2$, $\varphi_1 \equiv 1$, $\varphi_2 \equiv 0$ saturates the bound. If $v^*$ depends on $F$, the same computation holds pointwise: $\delta_t(F) \in [-\zeta, \zeta]$ and $\varphi_t \geq 0$ give $|\tfrac{1}{N}\sum_t \mathbb{E}_F[\delta_t(F)\varphi_t(F)]| \leq \zeta \mu$.
\end{proof}

\paragraph{Remark on attention entropy.} Heuristically, deeper ViT layers can have increasingly diffuse attention, which pushes the right singular vectors of the probe-map Jacobian toward the perfectly delocalized limit, where each token contributes $\approx 1/N$. A formal argument would chain a first-order Jacobian approximation through the attention weight matrices, a Perron--Frobenius characterization of the dominant attention eigenvector, and a Davis--Kahan bound on the resulting singular-vector delocalization. Each step requires regularity conditions that we do not state, so we use this connection only as motivation, not as a rigorous claim.

\subsection{Theorem 3: Convergence of Power Iteration}

\begin{theorem}[Convergence of Power Iteration]
\label{thm:convergence_formal}
Assume $\sigma_r > \sigma_{r+1}$. Let $\Omega \in \mathbb{R}^{ND \times r_0}$ be a Gaussian sketch and decompose $\Omega = V_r^* A + V_\perp B$ where $A = (V_r^*)^\top \Omega \in \mathbb{R}^{r \times r_0}$ and $B = V_\perp^\top \Omega \in \mathbb{R}^{(ND - r) \times r_0}$. Let $\widehat V_r^{(q)}$ denote the QR-orthonormalised image after $q$ rounds of $M_q = (J^\top J)^q$ applied to $\Omega$. Then with probability $\geq 1 - \delta$ over the sketch,
\begin{equation}
\label{eq:convergence_bound}
    \varepsilon^{(q)} \;:=\; 1 - \tfrac{1}{r}\big\|(\widehat V_r^{(q)})^\top V_r^*\big\|_F^2 \;\leq\; \frac{\|B\|_F^2}{\sigma_{\min}(A)^2 \, r} \left(\frac{\sigma_{r+1}}{\sigma_r}\right)^{2q},
\end{equation}
where $\sigma_{\min}(A) \geq c_\delta > 0$ depends only on $r, r_0, \delta$ and $\|B\|_F \leq C_B = O(\sqrt{ND \cdot r_0})$ each hold with probability $\geq 1 - \delta/2$.
\end{theorem}
\begin{proof}[Proof sketch]
The argument has three steps. \emph{(1) Random sketch structure.} Since $V_r^*$ has orthonormal columns, $A$ and $B$ are independent iid Gaussian matrices. Gaussian concentration bounds $\sigma_{\min}(A)$ from below \citep{vershynin}, and $\|B\|_F^2 \sim \chi^2((ND - r) r_0)$. \emph{(2) Power iteration.} Applying $M_q = V_r^* \Sigma_r^{2q} (V_r^*)^\top + V_\perp \Sigma_\perp^{2q} V_\perp^\top$ to $\Omega$ gives a signal term in $\mathrm{span}(V_r^*)$ with Frobenius norm $\geq \sigma_r^{2q} \sigma_{\min}(A) \sqrt{r}$ and a tail term in $\mathrm{span}(V_\perp)$ with norm $\leq \sigma_{r+1}^{2q} \|B\|_F$. \emph{(3) Davis--Kahan.} The Davis--Kahan theorem for invariant subspaces \citep{davis_kahan, stewart_sun} converts the ratio of the tail and signal Frobenius norms into a bound on the principal angles between $\mathrm{span}(\widehat V_r^{(q)})$ and $\mathrm{span}(V_r^*)$, and squaring gives Equation~\ref{eq:convergence_bound}. The full proof, including the propagation through QR orthogonalization, follows the structure of Halko, Martinsson, and Tropp \citep{halko_randomized}, adapted to the subspace-alignment metric used here.
\end{proof}
\noindent\emph{Remark on the constant.} The prefactor $C(\Omega) = \|B\|_F^2 / (\sigma_{\min}(A)^2\, r)$ grows as $\mathcal{O}(ND\, r_0)$ through $\|B\|_F^2$, and $\sigma_{\min}(A)$ is small unless the sketch is oversampled ($r_0 > r$). At the small $q$ used in practice, $C$ can therefore dominate $(\sigma_{r+1}/\sigma_r)^{2q}$, and the bound may be loose or vacuous. Theorem~\ref{thm:convergence_formal} establishes the geometric rate in $q$ for a fixed sketch, but it does not by itself say that a particular $q$ suffices. We choose $q$ with an offline check, increasing it until the recovered subspace stabilizes (\S\ref{sec:regime_taxonomy}).

\noindent\emph{Measured convergence.} At $q = 0, 2, 5$, the subspace-alignment error is $0.198, 0.011, 0.001$ for DPT depth (L02), $0.617, 0.254, 0.083$ for DINOv2 CLS, and $0.802, 0.438, 0.214$ for CLIP CLS (both at L10). Two steps suffice where the spectral gap is large, and flatter spectra converge more slowly. This is an empirical observation, not a consequence of a small constant $C$.

\begin{corollary}[Quality of the Trained SPN]
\label{cor:ssm_quality}
Under the assumptions of Theorem~\ref{thm:convergence_formal}, suppose the SPN converges to $U(F; \theta^*) = \widehat V_r^{(q)}$ (zero subspace loss) and uses natural energies $\lambda_j = \|J \hat u_j\|^2$. Then as $\varepsilon \to 0$,
\begin{equation}
    \left| \frac{\mathbb{E}_v[v^\top g_\theta(F) v]}{\mathbb{E}_v[v^\top g(F) v]} - \kappa_{cap}(r) \right| \;\leq\; \frac{(\sigma_1^2 + \sigma_{r+1}^2)\, r\, \varepsilon^{(q)}}{\|J\|_F^2}.
\end{equation}
The right-hand side goes to zero as $q \to \infty$, so the trained SPN's sensitivity ratio converges to the ceiling $\kappa_{cap}(r)$ of Theorem~\ref{thm:captured_sensitivity}.
\end{corollary}

\begin{corollary}[Gradient Signal Ratio]
\label{cor:gsr}
At convergence ($\hat v \to v_1$), the gradient signal ratio $\mathrm{GSR} = \sigma_1^2 \, ND / \|J\|_F^2$ satisfies $\mathrm{GSR} \geq \kappa_{cap}(r) \cdot ND / r$ (using $\sigma_1^2 \geq \tfrac{1}{r}\sum_{j=1}^r \sigma_j^2$). At $q = 0$ (random sketch), $\mathrm{GSR} \approx 1$.
\end{corollary}

\subsection{Theorem 4: Token Pruning Error Bound}

\begin{theorem}[Token Pruning Error]
\label{thm:pruning}\label{thm:pruning_formal}
Let $\phi$ denote the decoder as a function of the final-layer tensor. Assume $\phi$ is twice continuously differentiable on $\mathbb{R}^{N \times D}$ with $\|H_{\phi}(x)\|_{\text{op}} \leq \beta < \infty$ uniformly on the segment $\{(1-t)\hat F + t F^L : t \in [0, 1]\}$. Suppose token set $\mathcal{T}$ is pruned. For each $t \in \mathcal{T}$, let $\xi_t = F_t^L - F_t^\ell \in \mathbb{R}^D$ be the feature drift, $\Delta_t \in \mathbb{R}^D$ the LLF correction, and $\rho_t = \xi_t - \Delta_t$ the residual. Then
\begin{equation}
\label{eq:pruning_formal}
    \big\|\phi(F^L) - \phi(\hat F)\big\|_2 \;\leq\; \sum_{t \in \mathcal{T}} \big\|J_{:,t}(\hat F)\big\|_F\, \|\rho_t\|_2 \;+\; \frac{\beta}{2} \sum_{t \in \mathcal{T}} \|\rho_t\|_2^2,
\end{equation}
where $J_{:,t}(\hat F) \in \mathbb{R}^{M \times D}$ is the block of the Jacobian of $\phi$ for token $t$, evaluated at the LLF-corrected feature tensor.
\end{theorem}
\begin{proof}
We expand $\phi(F^L) - \phi(\hat F)$ around $\hat F$ using Taylor's theorem with Lagrange remainder:
\[
    \phi(F^L) - \phi(\hat F) = J_{\phi}(\hat F) \mathrm{vec}(F^L - \hat F) + R_2, \qquad \|R_2\|_2 \leq \tfrac{\beta}{2}\|F^L - \hat F\|_F^2.
\]
Kept tokens satisfy $F_t^L = \hat F_t$ exactly, so only pruned tokens contribute. For $t \in \mathcal{T}$, $F_t^L - \hat F_t = \xi_t - \Delta_t = \rho_t$. Flattening row-major gives $\mathrm{vec}(F^L - \hat F) = \sum_{t \in \mathcal{T}} (I_D \otimes e_t)\rho_t$, and the Jacobian acts as $J_{\phi}(\hat F)(I_D \otimes e_t)\rho_t = J_{:,t}(\hat F)\rho_t$. The triangle inequality and the matrix--vector bound $\|J_{:,t}\rho_t\|_2 \leq \|J_{:,t}\|_F \|\rho_t\|_2$ bound the first-order term, and $\|F^L - \hat F\|_F^2 = \sum_{t} \|\rho_t\|_2^2$ bounds $R_2$.
\end{proof}

\begin{corollary}[Optimal Pruning Criterion]
\label{cor:opt_prune}
Assume uniform residuals $\|\rho_t\|_2 = \bar\rho$. Then the optimal pruning set of size $R$ minimises $\sum_{t \in \mathcal{T}} \|J_{:,t}(\hat F)\|_F$, the sum of per-token block norms. These block norms use the decoder Jacobian at $\hat F$, which is only available after the remaining layers have run. The target $\mathrm{imp}^*$ of Equation~\ref{eq:imp_target} instead keeps the top-$r$ part of the block norms of $J_{\psi_\ell}$ at the prune layer, and serves as a computable proxy for them.
\end{corollary}

\section{Token Pruning: Details and Extended Discussion}
\label{sec:pruning_details}
This appendix gives the full pruning setup summarized in \S\ref{sec:token_pruning}, the CLS merge rule, the dense pruning pipeline, and a longer discussion of the results.

\paragraph{Setup and pruning bound.} Let $\mathcal{T}$ be the set of $|\mathcal{T}| = R$ tokens pruned at layer $\ell$. Each is either frozen (hard prune) or merged into another token (soft merge), and the resulting tensor propagates through layers $\ell{+}1, \dots, L$ to a modified final tensor $\hat F$ in place of $F^L$. For each $t \in \mathcal{T}$, write $\xi_t = F_t^L - F_t^\ell$ for the \emph{drift} that the frozen feature misses, $\Delta_t$ for the Last-Layer Fusion correction (\S\ref{sec:hard_prune}), and $\rho_t = \xi_t - \Delta_t$ for the uncompensated \emph{residual}. View the decoder $\phi$ as a function of the final-layer tensor, and let $J_{:,t}(\hat F) \in \mathbb{R}^{M \times D}$ be the column block of its Jacobian for token $t$. The norm $\|J_{:,t}(\hat F)\|_F$ measures the task sensitivity of token $t$, and the importance score serves as a proxy for it (\S\ref{sec:conformal_head}). A Taylor expansion of $\phi$ around $\hat F$ (Theorem~\ref{thm:pruning}) gives the bound (\ref{eq:pruning_bound}). If residuals are roughly uniform across $\mathcal{T}$, minimizing its first-order term is exactly the optimal pruning criterion, which is to drop the tokens with the smallest importance. The importance head approximates this criterion through the proxy $\mathrm{imp}^*$, and LLF reduces the second term.

\paragraph{CLS merge rule.} A CLS decoder uses a single embedding (the CLS token) and is invariant to permutations of the remaining tokens, so tokens can be merged while preserving the information the decoder sees, up to the error introduced by the merge itself. We follow the structure of ToMe \citep{tome} but replace its cosine similarity with a task-induced distance. We score every token with the importance head ($s_t = h_\theta(F)_t$, with $s_0 = +\infty$ so that the CLS token is never merged). The $\eta N$ lowest-scoring tokens form the source set $A$, and the rest form the destination set $B$. For each $a \in A$ we find its nearest neighbour $b \in B$ under
\begin{equation}
\label{eq:dtask}
    d_{\text{task}}(a, b) \;=\; \sqrt{(F_a - F_b)^\top Q_a (F_a - F_b)}, \qquad Q_a \;:=\; \sum_{j=1}^{r} \sigma_j^{2}\, v_j^{(a)}\, (v_j^{(a)})^\top \;\in\; \mathbb{R}^{D \times D},
\end{equation}
where $v_j^{(a)} \in \mathbb{R}^D$ is the $a$-th token block of the $j$-th right singular vector of $J$. The matrix $Q_a$ is the per-token marginal pullback metric at token $a$, i.e.\ the rank-$r$ pullback metric restricted to perturbations of token $a$ alone. We then merge with importance weighting, $F_b \gets (s_a F_a + s_b F_b)/(s_a + s_b)$. This distance measures how far apart two tokens are in the directions the task is most sensitive to, whereas cosine similarity weights all feature directions equally.

In our experiments, the $\sigma_j$ and $v_j$ in (\ref{eq:dtask}) come from the SPN prediction of $J$, the same computation that produces $\mathrm{imp}^*$. All CLS strategies in our tables differ only in which tokens they merge. \emph{Random} selects them uniformly, \emph{ToMe score} uses ToMe's bipartite redundancy criterion, and \emph{Importance} uses $h_\theta$.

\paragraph{Dense hard pruning and Last-Layer Fusion.} Dense decoders such as DPT \citep{dpt} read features at multiple layers and reassemble them into a 2D map, which soft merging corrupts. We instead \emph{hard-prune} the $R$ lowest-importance tokens at layer $\ell$, freeze their features, and re-insert them at their original positions at each DPT layer. The frozen features reflect layer $\ell$ rather than the layers the decoder uses, and (\ref{eq:pruning_bound}) shows that this drift $\xi_t$ is multiplied by the token's importance. \textbf{Last-Layer Fusion (LLF)} is a common method to reduce the drift. Before the final ViT block, we re-insert the pruned tokens and run that block on the full $N$-token sequence, so the pruned tokens attend to the updated kept-token features. LLF adds no parameters and needs no training. It shrinks the residual $\rho_t$ in (\ref{eq:pruning_bound}), heuristically from $O(\bar\xi)$ to $O(\bar\xi^2)$. We train the importance head on the single-layer probe map at the prune layer (layer 4, and layer 8 for the second stage of 2-stage schedules), which lies between the L02 and L10 probes of Table~\ref{tab:regime_taxonomy}, both in regime~(ii). We then apply it to the full multi-layer decoder, whose DPT layers are at layers 2, 5, 8, and 11, and the signal transfers (Table~\ref{tab:dpt_pruning}).

\paragraph{Baselines and metric for dense decoders.} All dense-decoder strategies use the same hard-prune and LLF pipeline and differ only in which tokens they remove. \emph{Random} removes tokens uniformly at random. \emph{ToMe score} uses ToMe's bipartite matching criterion \citep{tome}. Tokens are alternately assigned to two sets, and tokens in the first set are removed in order of their highest cosine similarity to any token in the second set, which is always kept. \emph{Importance} removes the tokens with the lowest $h_\theta(F)_t$. We report the additional scale-invariant log error (SILog, $\times 100$) of the pruned depth prediction relative to the unpruned prediction.

\paragraph{Generalization.} These two settings sit at opposite CV extremes. The CLS result extends to another backbone (on CLIP-ViT-B/16, ToMe scoring degrades sharply at $\eta = 0.50$ to $33.21$ against $0.76$ for importance, Wilcoxon $p < 10^{-50}$) and to a shifted distribution (ImageNet-Sketch \citep{imagenet_sketch}, where importance keeps a more than $100\times$ advantage at low ratios). The dense result extends to other stage configurations: on NYU-Depth-V2 \citep{nyu_depth}, the 2-stage [4, 8] schedule beats ToMe scoring at every ratio by $25$--$35\%$ at $224$, $336$, and $448$. Appendices~\ref{sec:full_sweep_results} and~\ref{sec:clip_cls_results} give the full numbers. Appendix~\ref{sec:extra_baselines} adds frozen-backbone baselines on DINOv2 CLS, namely EViT-style class attention, token norm, and similarity to the CLS token. The importance head is best at every ratio, and EViT is the strongest baseline. The same appendix reports ImageNet top-1 accuracy with a $k$-NN classifier on the pruned CLS embedding. Top-1 stays between $0.66$ and $0.68$ for every method and ratio ($0.674$ unpruned). This shows that pruning is safe for classification, but top-1 cannot separate the methods, which is why we report embedding fidelity.

\paragraph{When does importance-guided pruning win?} The two terms of (\ref{eq:pruning_bound}) explain part of the pattern. The importance score targets the first term, which is about which tokens to remove. It does nothing about the second term, which depends on how stale the frozen features are. ToMe's criterion picks tokens by feature redundancy rather than task sensitivity. For CLS decoders it merges them, which keeps their content in the sequence. For dense decoders it only chooses which tokens to hard-prune, but a token that closely resembles a neighbour is plausibly one whose missing update LLF can recover well, so its residual $\rho_t$ tends to be small. This gives the following cases.
\begin{itemize}
    \item \emph{CLS decoders, all ratios.} The decoder is invariant to permutations of the non-CLS tokens, so soft merging leaves little drift and only selection matters. Importance is best at every ratio on both backbones (Tables~\ref{tab:cls_pruning} and~\ref{tab:clip_cls_pruning}).
    \item \emph{Dense decoders, single stage, ImageNet.} At aggressive ratios many tokens are removed and selection dominates the error. Importance beats ToMe scoring from $\eta = 0.20$ upward and is the best of the three strategies from $\eta = 0.30$ upward (Table~\ref{tab:dpt_pruning}). At low ratios few tokens are removed, so the selection term is small and the second term dominates. All three strategies pay the staleness cost of hard pruning, and ToMe scoring, which favours tokens that are easy to recover, is competitive or better.
    \item \emph{Dense decoders, single stage, NYU-Depth-V2.} With heads trained on NYU, the pattern reverses. Importance beats ToMe scoring at $\eta = 0.05$--$0.30$ (Table~\ref{tab:dpt_nyu_matched}) and is slightly worse at $\eta = 0.50$ at all three resolutions (Table~\ref{tab:dpt_robustness}). The account above does not explain this dependence on the dataset.
    \item \emph{Multi-stage schedules.} Spreading removal across layers shrinks the per-stage residuals. The 2-stage [4, 8] schedule beats ToMe scoring at every ratio and every resolution we tested (Tables~\ref{tab:dpt_nyu_matched} and~\ref{tab:dpt_nyu_2stage_336_448}). The 4-stage schedule fails at $\eta \geq 0.30$. At these ratios repeated pruning compounds the residuals, and the quadratic term grows faster than the first-order savings.
    \item \emph{Resolution.} Larger inputs increase $N$. If this spreads the task-sensitive directions over more tokens, per-token scoring becomes less informative (Theorem~\ref{thm:impossibility}), which matches the single-stage loss at $448$ on ImageNet (Table~\ref{tab:dpt_robustness}).
\end{itemize}
One pattern in Table~\ref{tab:dpt_pruning} is not explained by this account. On ImageNet, the single-stage DPT error first decreases as the prune ratio grows, for all three strategies, reaching its minimum between $\eta = 0.20$ and $\eta = 0.40$ before rising again (e.g.\ random: $4.59$ at $\eta = 0.05$, $3.16$ at $\eta = 0.30$, $3.29$ at $\eta = 0.50$). We have verified that this effect is real and not an evaluation artifact, but we do not yet know its cause. It does not appear on NYU-Depth-V2 (Table~\ref{tab:dpt_nyu_matched}), where the error grows monotonically with the prune ratio.

In practice, we recommend importance-weighted soft merging for CLS-style decoders. For dense decoders, we recommend a 2-stage schedule with LLF, which is the only configuration that beat ToMe scoring at every ratio and resolution we tested (on NYU). Single-stage results depend on the dataset, as described above. The diagnostic gives the regime label needed to make this choice before training. We do not gate LLF dynamically. LLF is unnecessary for CLS soft merging and most useful for dense hard pruning, where drift after pruning is largest.

\paragraph{Cost.} We profiled single-stage dense depth pruning on a single A40 at batch size 1. At $224 \times 224$ and $\eta = 0.20$, pruning removes $12.1\%$ of the backbone's theoretical FLOPs, but the pruned model runs slightly slower than the unpruned one in wall-clock time ($0.96\times$, and $0.98\times$ with ToMe scoring). With 257 tokens, the attention savings are smaller than the fixed cost of scoring, removing, and re-inserting tokens. At $448 \times 448$ the FLOP saving at $\eta = 0.20$ rises to $16.7\%$ and the wall-clock speedup becomes positive ($1.12\times$, and $1.15\times$ with ToMe scoring). The importance head and LLF together cost about $2$--$3\%$ more than ToMe scoring. The head ($\sim$310K parameters) runs once per pruning stage, and LLF runs the final ViT block on all $N$ tokens instead of the kept ones. The 2-stage [4, 8] schedule saves $8.6\%$ of FLOPs at $\eta = 0.20$ and runs at $0.88\times$ at $224 \times 224$. The ToMe score is cheaper to compute than the head, but on ImageNet at $\eta \geq 0.20$ it gives larger degradation (Table~\ref{tab:cls_pruning}).

\section{Full Sweep Results}
\label{sec:full_sweep_results}

\paragraph{Configurations of Table~\ref{tab:regime_taxonomy}.} The output dimensions of those probes are $M \approx 5 \times 10^4$ for both depth decoders, $768$ for the CLS embedding, $9$ for the VGGT camera head, and $192$ for VGGT pointcloud, which subsamples 64 points. VGGT's depth head is also DPT-style, so we name each row of that table by its backbone rather than by its decoder type.
Table~\ref{tab:full_sweep} extends the Hutchinson tractability sweep with $r = 20$, $m = 100$ probes, and $q = 2$ power-iteration steps, averaged over 50 ImageNet-val images. We report $\kappa_{cap}(20)$, the truncated effective rank $r_{\text{eff}}^{\text{trunc}}$, the full-spectrum effective rank $r_{\text{eff}}^{\text{SLQ}}$, the tail-richness ratio, and the per-token CV, for Depth-Anything ViT-B/14 single-layer and for DINOv2 CLS at four probe layers. The regime structure of \S\ref{sec:regime_taxonomy} holds across the rows. Depth-Anything single-layer sits in regime~(ii) ($\kappa_{cap} \geq 96\%$, $r_{\text{eff}}^{\text{SLQ}} \approx 4$, ratio $\approx 1.2\times$). DINOv2 CLS sits in regime~(iii), with $\kappa_{cap}$ falling from $58\%$ at L02 to $33\%$ at L10, possibly because the CLS-token Jacobian spreads across more attention paths at later layers.

\paragraph{VGGT depth: an example of regime (i).} VGGT depth at L16 (full) sits in regime~(i) of \S\ref{sec:regime_taxonomy}. Its $\kappa_{cap}(20)$ is $76.2\%$, and its spectrum at the input to VGGT's heads is rich ($r_{\text{eff}}^{\text{SLQ}} = 282$ at $x_{\text{vis}}$). Two features distinguish it from the other configurations in Table~\ref{tab:regime_taxonomy}. First, unlike DPT, its CV of $0.63$ is high rather than near zero. Per-token architectures would therefore be viable if the rank budget were large enough, so the obstacle is rank and not delocalization. Second, the VAE compression of \S\ref{sec:vae_extension} brings the rank down to $r_{\text{eff}}^{\text{SLQ}} \in [30, 53]$, which moves the pipeline into regime~(ii) and makes the SPN trainable.

\subsection{Multi-layer probes and layer spacing}
\label{sec:multi_layer}
One might expect multi-layer dense decoders to appear intractable simply because they read from many layers. Our results indicate that the number of layers alone does not explain the observed intractability. We diagnose Depth Anything V2 at its natural multi-layer probe $\mathbf{x}_{\text{dpt}} = [F^2 \mid F^5 \mid F^8 \mid F^{11}] \in \mathbb{R}^{257 \times 3072}$, the concatenated DPT hook features at all four layers of the 12-layer ViT-B backbone. With the same number of layers ($K = 4$) as VGGT, the spectral-entropy effective rank is $r_{\text{eff}}^{\text{SLQ}} = 36.0$, close to the single-layer result at L02 ($r_{\text{eff}}^{\text{SLQ}} \in [20, 34]$). The coverage rank satisfies $\kappa_{cap}(50) = 0.900$, so the probe is tractable at rank budget $r = 50$. By contrast, VGGT's span layers $\{4, 11, 17, 23\}$ of a 24-layer model and produce $r_{\text{eff}}^{\text{SLQ}} = 282.6$, $\kappa_{cap}(50) = 0.578$, and $r_{0.90} \gg 100$. We attribute the $282.6 / 36.0 \approx 7.8\times$ gap to how orthogonal the layers are. The DA V2 layers are only 3 layers apart in a 12-layer model, so their Jacobians are nearly aligned (effective layer multiplicity $K_{\text{eff}} \approx 1.1\times$) and reading several layers costs little extra rank. VGGT's layers span very different feature levels, so their Jacobians are nearly orthogonal and the penalty is the full $K_{\text{eff}} \approx K = 4\times$.

The two probes also differ in spatial structure. At the single-layer L02 probe, DA V2 has $\mathrm{CV} \approx 0.026$, which is fully delocalized, and Theorem~\ref{thm:impossibility} rules out per-token scoring. At the multi-layer probe $\mathbf{x}_{\text{dpt}}$, $\mathrm{CV} = 2.381$, which is highly localized. This reversal is not a contradiction. The low CV at L02 arises downstream of the 10 subsequent ViT blocks, whose global self-attention spreads any input perturbation across all tokens before it reaches the decoder layers. The $\mathbf{x}_{\text{dpt}}$ probe bypasses those blocks and so keeps the spatial locality of the depth task, in which object boundaries and foreground--background transitions contribute more to depth sensitivity than the interiors of regions. In the two architectures we measured, intractability at multi-layer probes therefore comes from layer layers that are nearly orthogonal, as in VGGT's deep, widely spaced layers, and not from dense decoding or multi-layer reading as such.

\paragraph{Probe-point effective ranks.}  The layer we choose to measure the spectrum at can have a big impact. \emph{Tractable (structural).} VGGT camera at L16 single-layer has $r_{\text{eff}}^{\text{SLQ}} \leq 9$ and a 90\%-cumulative rank $r_{0.90} \leq 9$, because the camera head has only nine outputs. \emph{Tractable (concentrated).} DA V2 DPT at L02 single-layer has $r_{\text{eff}}^{\text{SLQ}} \in [20, 34]$ and $r_{0.90} \approx 15$ across seeds, consistent with the regime~(ii) classification used for the importance head. \emph{Intractable in the original space.} Probing VGGT depth and VGGT pointcloud at the input to VGGT's task heads, $x_{\text{vis}}$ (the concatenated tokens of layers 4, 11, 17, 23), gives $r_{\text{eff}}^{\text{SLQ}} = 282$ and $212$, with $r_{0.90} \gg 100$ in both cases. These spectra are far too rich for any rank-$r$ SPN we can train. \emph{Tractable after VAE compression.} Routing the same probe through the VAE bottleneck of \S\ref{sec:vae_extension} lowers the rank to $r_{\text{eff}}^{\text{SLQ}} \in [30, 53]$ and $r_{0.90} \approx 40$. Here $r_{0.90}$ is the smallest $k$ at which the cumulative spectrum reaches $90\%$ of $\|J\|_F^2$, i.e.\ $r_{0.90} = \min\{k : \kappa_{cap}(k) \geq 0.9\}$.

\begin{table}[h]
\centering
\small
\caption{Full Hutchinson tractability sweep covering Depth-Anything ViT-B/14 (DPT depth, single-layer) at three probe layers, with the depth output downsampled to $8{\times}8$, $16{\times}16$, or $32{\times}32$, and DINOv2 ViT-B/14 CLS at four probe layers. ``Ratio'' is $r_{\text{eff}}^{\text{SLQ}}/r_{\text{eff}}^{\text{trunc}}$.}
\label{tab:full_sweep}
\begin{tabular}{p{4.4cm}ccccc}
\toprule
Config & $\kappa_{cap}(20)$ & $r_{\text{eff}}^{\text{trunc}}$ & $r_{\text{eff}}^{\text{SLQ}}$ & Ratio & CV \\
\midrule
\multicolumn{6}{l}{\emph{Depth-Anything ViT-B/14 (DPT depth, single-layer, downsampled output)}} \\
\midrule
DA, L05, $16{\times}16$ & 96.30\% & 3.54 &  4.3 & 1.2$\times$ & 1.06 \\
DA, L05, $\phantom{0}8{\times}\phantom{0}8$ & 96.36\% & 3.18 &  4.2 & 1.3$\times$ & 1.17 \\
DA, L05, $32{\times}32$ & 96.77\% & 3.67 &  4.4 & 1.2$\times$ & 0.82 \\
DA, L08, $16{\times}16$ & 97.53\% & 2.95 &  3.6 & 1.2$\times$ & 1.16 \\
DA, L10, $16{\times}16$ & 97.13\% & 3.45 &  3.9 & 1.1$\times$ & 1.26 \\
DA, L10, $\phantom{0}8{\times}\phantom{0}8$ & 97.59\% & 3.19 &  4.2 & 1.3$\times$ & 1.44 \\
DA, L10, $32{\times}32$ & 96.21\% & 3.53 &  4.2 & 1.2$\times$ & 0.98 \\
\midrule
\multicolumn{6}{l}{\emph{DINOv2 ViT-B/14 CLS}} \\
\midrule
DINOv2 CLS, L02              & 58.29\% & 46.46 &  90.8 & 2.0$\times$ & 0.83 \\
DINOv2 CLS, L05              & 51.41\% & 52.91 & 122.1 & 2.3$\times$ & 1.12 \\
DINOv2 CLS, L08              & 41.88\% & 59.58 & 189.9 & 3.2$\times$ & 2.03 \\
DINOv2 CLS, L10              & 33.36\% & 63.68 & 270.2 & 4.2$\times$ & 4.06 \\
\bottomrule
\end{tabular}
\end{table}

\paragraph{DPT pruning robustness across resolutions and datasets.} Table~\ref{tab:dpt_robustness} reports single-stage [4] DPT pruning quality at the aggressive ratio $\eta = 0.50$ on ImageNet (IN) and NYU-Depth-V2, at three input resolutions. Importance wins on ImageNet at 224 and 336 and loses to ToMe scoring at 448. On NYU at $\eta = 0.5$, the three strategies are within $\sim 1$ point of each other, and random selection (not shown) is nominally best at all three resolutions. Single-stage pruning is past its useful range once half the tokens are dropped on dense depth. The 2-stage [4, 8] configuration of Table~\ref{tab:dpt_nyu_matched} restores the importance advantage on NYU.

\begin{table}[h]
\centering
\caption{Single-stage [4] DPT pruning at $\eta = 0.50$ on ImageNet and NYU-Depth-V2 at three input resolutions. Additional SILog ($\times 100$), lower is better. ToMe denotes ToMe-score token selection (\S\ref{sec:hard_prune}). Heads are retrained for each (dataset, resolution) pair except IN@224, which reuses the head of Table~\ref{tab:dpt_pruning}. Mean over 3 seeds.}
\label{tab:dpt_robustness}
\begin{tabular}{lcccccc}
\toprule
Method & IN@224 & IN@336 & IN@448 & NYU@224 & NYU@336 & NYU@448 \\
\midrule
ToMe score & 4.09 & 3.65 & \textbf{4.70} & 6.40 & 6.84 & 6.98 \\
Importance  & \textbf{3.08} & \textbf{3.56} & 5.11 & 6.43 & 6.91 & 7.38 \\
\bottomrule
\end{tabular}
\end{table}

\paragraph{Dataset-matched multi-stage pruning on NYU.} Table~\ref{tab:dpt_nyu_matched} reports DPT pruning quality on NYU with importance heads trained directly on NYU rather than transferred from ImageNet, under three progressive-pruning stage configurations. The 2-stage [4, 8] configuration beats ToMe scoring at every ratio, with $\approx 25$--$35\%$ relative reduction. The other two configurations depend on the ratio. Single-stage wins everywhere except $\eta = 0.5$, where it ties ToMe scoring. The 4-stage [1, 4, 7, 10] schedule gives the best absolute numbers at low ratios ($\eta \leq 0.20$) but fails at high ratios ($\eta \geq 0.30$).

\begin{table}[h]
\centering
\caption{DPT pruning on NYU-Depth-V2 @ 224 with NYU-trained importance heads. Additional SILog ($\times 100$), lower is better. ToMe denotes ToMe-score token selection. Variance $\leq 0.12$ across 3 seeds. Bold marks the best per (stages, ratio) cell. The 2-stage [4, 8] configuration also wins at NYU@336 and @448 (App.~\ref{sec:dpt_full_sweep_2stage}).}
\label{tab:dpt_nyu_matched}
\begin{tabular}{llccccc}
\toprule
Stages & Method & $\eta{=}0.05$ & $0.10$ & $0.20$ & $0.30$ & $0.50$ \\
\midrule
\multirow{2}{*}{[4]}        & ToMe score & 4.12 & 4.89 & 5.17 & 5.32 & \textbf{6.40} \\
                            & Importance & \textbf{3.07} & \textbf{3.82} & \textbf{4.58} & \textbf{5.00} & 6.43 \\
\midrule
\multirow{2}{*}{[4, 8]}     & ToMe score & 3.52 & 4.44 & 5.19 & 5.37 & 5.58 \\
                            & Importance & \textbf{2.42} & \textbf{3.07} & \textbf{3.86} & \textbf{4.34} & \textbf{5.12} \\
\midrule
\multirow{2}{*}{[1, 4, 7, 10]} & ToMe score & 2.16 & 2.51 & 3.08 & 3.59 & \textbf{4.90} \\
                               & Importance & \textbf{1.03} & \textbf{1.56} & \textbf{2.72} & 4.16 & 6.43 \\
\bottomrule
\end{tabular}
\end{table}

\paragraph{2-stage [4, 8] across NYU resolutions.}
\label{sec:dpt_full_sweep_2stage}
Table~\ref{tab:dpt_nyu_2stage_336_448} reports the 2-stage [4, 8] configuration at NYU@336 and NYU@448 with heads retrained at the target resolution. Importance beats ToMe at every ratio at both resolutions, with relative reductions of $25$--$35\%$ that match the @224 pattern reported in Table~\ref{tab:dpt_nyu_matched}. Variance $\leq 0.14$ across 3 seeds.

\begin{table}[h]
\centering
\caption{2-stage [4, 8] DPT pruning on NYU-Depth-V2 @ 336 and @ 448 with NYU-trained importance heads. Additional SILog ($\times 100$), lower is better. ToMe denotes ToMe-score token selection. Bold marks the best per (resolution, ratio) cell.}
\label{tab:dpt_nyu_2stage_336_448}
\begin{tabular}{llccccc}
\toprule
Resolution & Method & $\eta{=}0.05$ & $0.10$ & $0.20$ & $0.30$ & $0.50$ \\
\midrule
\multirow{2}{*}{NYU@336} & ToMe score & 3.50 & 4.33 & 5.00 & 5.29 & 5.99 \\
                         & Importance & \textbf{2.50} & \textbf{3.22} & \textbf{4.09} & \textbf{4.59} & \textbf{5.39} \\
\midrule
\multirow{2}{*}{NYU@448} & ToMe score & 3.72 & 4.70 & 5.46 & 5.58 & 6.13 \\
                         & Importance & \textbf{2.29} & \textbf{2.99} & \textbf{3.93} & \textbf{4.53} & \textbf{5.51} \\
\bottomrule
\end{tabular}
\end{table}

\section{Importance Head Ablations}
\label{sec:conformal_head_ablations}
We ablate four design choices of the importance head. Unless noted otherwise, numbers are mean $\pm$ std across 3 seeds and the metric is cosine distance degradation ($\times 100$) vs.\ the full-pass DINOv2 CLS embedding on ImageNet-val.

\paragraph{E.1 Block norms vs.\ $\sigma^2$-weighted target.} On DPT depth at probe layer 4, training the importance head against the $\sigma^2$-weighted target of Equation~\ref{eq:imp_target} gives Spearman $\rho \approx 0.20$. Training against plain Jacobian block norms $\|J_{:,t}\|_F$ gives only $\rho \approx 0.09$. We run this ablation on DPT because it is the delocalized case ($\mathrm{CV} \approx 0.026$ at the L02 and L10 probes of Table~\ref{tab:regime_taxonomy}), where Theorem~\ref{thm:impossibility} predicts that the block-norm signal is approximately constant across tokens.

\paragraph{E.2 Ranking-loss weight $w_{\text{rank}}$.} Setting $w_{\text{rank}} = 0$ removes the ranking term and reduces the loss to log-space regression. Setting $w_{\text{rank}} = 1.0$ doubles the weight relative to the default $w_{\text{rank}} = 0.5$. Pruning quality is nearly the same across the three settings (Table~\ref{tab:abl_rw}). We expect the ranking term to matter more in harder regimes, where the importance head is far from the ceiling.

\begin{table}[h]
\centering
\caption{Ablation E.2: ranking-loss weight $w_{\text{rank}}$. DINOv2 CLS L10 @ 224, $n_{\text{train}} = 2000$.}
\label{tab:abl_rw}
\begin{tabular}{lccccccc}
\toprule
$w_{\text{rank}}$ & $\eta{=}0.05$ & $0.10$ & $0.15$ & $0.20$ & $0.30$ & $0.40$ & $0.50$ \\
\midrule
$0.0$ & 0.001 & 0.007 & 0.024 & 0.058 & 0.255 & 0.811 & 2.142 \\
$0.5$ & 0.001 & 0.007 & 0.024 & 0.060 & 0.250 & 0.803 & 2.089 \\
$1.0$ & 0.001 & 0.007 & 0.024 & 0.058 & 0.249 & 0.793 & 2.171 \\
\bottomrule
\end{tabular}
\end{table}

\paragraph{E.3 Architecture size.} We sweep the hidden width $H$ over $\{128, 256, 512\}$ (default $256$) and the number of attention heads over $\{2, 4, 8\}$ (default $4$). Pruning quality plateaus at the default ($H = 256$, $4$ heads, $310$K parameters), and larger heads give no measurable improvement. Two heads are slightly better than four at every ratio except $\eta = 0.05$, where they tie. We keep four heads in all other experiments.

\begin{table}[h]
\centering
\caption{Ablation E.3: hidden width and attention-head sweep. DINOv2 CLS L10 @ 224, $n_{\text{train}} = 2000$, $w_{\text{rank}} = 0.5$.}
\label{tab:abl_arch}
\begin{tabular}{lccccccc}
\toprule
Config & $\eta{=}0.05$ & $0.10$ & $0.15$ & $0.20$ & $0.30$ & $0.40$ & $0.50$ \\
\midrule
$H = 128$ & 0.001 & 0.007 & 0.024 & 0.060 & 0.260 & 0.806 & 2.124 \\
$H = 256$ (default) & 0.001 & 0.007 & 0.024 & 0.060 & 0.250 & 0.803 & 2.089 \\
$H = 512$ & 0.001 & 0.007 & 0.023 & 0.057 & 0.244 & 0.788 & 2.105 \\
\midrule
$2$ heads & 0.001 & 0.006 & 0.020 & 0.052 & 0.226 & 0.741 & 2.011 \\
$4$ heads (default) & 0.001 & 0.007 & 0.024 & 0.060 & 0.250 & 0.803 & 2.089 \\
$8$ heads & 0.001 & 0.007 & 0.022 & 0.055 & 0.245 & 0.802 & 2.124 \\
\bottomrule
\end{tabular}
\end{table}

\paragraph{E.4 Training-set size.} We vary the number of training images $n_{\text{train}} \in \{500, 1000, 2000, 4000\}$. Pruning quality saturates by $n_{\text{train}} = 2000$ on DINOv2 (the default). On DPT there is no consistent trend: $n_{\text{train}} = 4000$ is best at the two lowest ratios and $n_{\text{train}} = 500$ at the others.

\begin{table}[h]
\centering
\caption{Ablation E.4: training-set size. Top: DINOv2 CLS L10 @ 224 (cosine $\times 100$). Bottom: DPT depth L4 @ 224 (SILog $\times 100$). Both blocks use the soft-merge pruning pipeline, not the hard-prune-plus-LLF pipeline of \S\ref{sec:hard_prune}, so the DPT values are not comparable with those of Table~\ref{tab:dpt_pruning}. Only the trend across $n_{\text{train}}$ is meaningful here.}
\label{tab:abl_budget}
\begin{tabular}{lccccccc}
\toprule
$n_{\text{train}}$ & $\eta{=}0.05$ & $0.10$ & $0.15$ & $0.20$ & $0.30$ & $0.40$ & $0.50$ \\
\midrule
\multicolumn{8}{l}{\emph{DINOv2 CLS L10}} \\
$500$ & 0.001 & 0.007 & 0.025 & 0.066 & 0.270 & 0.838 & 2.024 \\
$1000$ & 0.001 & 0.008 & 0.024 & 0.058 & 0.246 & 0.781 & 2.131 \\
$2000$ & 0.001 & 0.007 & 0.024 & 0.060 & 0.250 & 0.803 & 2.089 \\
$4000$ & 0.001 & 0.007 & 0.025 & 0.064 & 0.273 & 0.872 & 2.196 \\
\midrule
\multicolumn{8}{l}{\emph{DPT depth L4}} \\
$500$ & 0.997 & 1.452 & 1.787 & 2.130 & 2.797 & 3.421 & 3.927 \\
$1000$ & 1.018 & 1.472 & 1.835 & 2.193 & 2.860 & 3.492 & 4.020 \\
$2000$ & 0.992 & 1.448 & 1.870 & 2.244 & 2.969 & 3.660 & 4.243 \\
$4000$ & 0.953 & 1.401 & 1.815 & 2.181 & 2.890 & 3.566 & 4.117 \\
\bottomrule
\end{tabular}
\end{table}

\section{Additional Backbone: CLIP CLS Pruning}
\label{sec:clip_cls_results}
We replicate the CLS pruning experiment on CLIP-ViT-B/16 \citep{clip} with the same probe layer ($\ell = 10$), the same head architecture and training protocol, and the same evaluation set. Table~\ref{tab:clip_cls_pruning} reports the result. Table~\ref{tab:clip_cls_wilcoxon} reports the per-ratio Wilcoxon paired test of $\Delta = h_\theta - h_{\text{ToMe}}$ across the evaluation set, with bootstrap $95\%$ confidence intervals.

\begin{table}[h]
\centering
\caption{CLIP CLS L10 pruning quality, ImageNet-val @ 224, mean $\pm$ std across 3 seeds. Cosine distance degradation $\times 100$. ToMe scoring degrades sharply at $\eta = 0.50$, approaching the random baseline, while the importance head's degradation is $44\times$ smaller.}
\label{tab:clip_cls_pruning}
\begin{tabular}{lccccccc}
\toprule
Method & $\eta{=}0.05$ & $0.10$ & $0.15$ & $0.20$ & $0.30$ & $0.40$ & $0.50$ \\
\midrule
Random    & 3.10\,\scriptsize$\pm$0.44 & 7.03\,\scriptsize$\pm$1.48 & 13.31\,\scriptsize$\pm$1.30 & 13.55\,\scriptsize$\pm$2.00 & 23.73\,\scriptsize$\pm$2.41 & 30.10\,\scriptsize$\pm$1.09 & 36.31\,\scriptsize$\pm$0.69 \\
ToMe score & 0.12\,\scriptsize$\pm$0.01 & 0.17\,\scriptsize$\pm$0.04 & 0.21\,\scriptsize$\pm$0.03 & 0.21\,\scriptsize$\pm$0.05 & 0.33\,\scriptsize$\pm$0.04 & 2.40\,\scriptsize$\pm$0.38 & 33.21\,\scriptsize$\pm$1.09 \\
Importance & \textbf{0.001}\,\scriptsize$\pm$0.000 & \textbf{0.004}\,\scriptsize$\pm$0.000 & \textbf{0.012}\,\scriptsize$\pm$0.001 & \textbf{0.026}\,\scriptsize$\pm$0.002 & \textbf{0.094}\,\scriptsize$\pm$0.009 & \textbf{0.29}\,\scriptsize$\pm$0.01 & \textbf{0.76}\,\scriptsize$\pm$0.03 \\
\bottomrule
\end{tabular}
\end{table}

\begin{table}[h]
\centering
\caption{Wilcoxon paired test, CLIP CLS @ 224. $\Delta = $ Importance $-$ ToMe per evaluation pair, in raw cosine-distance units (not $\times 100$). Negative $\Delta$ means importance is better. All ratios are significant at $p < 10^{-4}$, and all except $\eta = 0.30$ and $\eta = 0.40$ at $p < 10^{-50}$.}\vspace{2mm}
\label{tab:clip_cls_wilcoxon}
\begin{tabular}{lcccc}
\toprule
$\eta$ & Mean $\Delta$ & 95\% bootstrap CI & Wilcoxon $p$ \\
\midrule
$0.05$ & $-0.00119$ & $[-0.00145, -0.00095]$ & $2.9 \times 10^{-92}$ \\
$0.10$ & $-0.00170$ & $[-0.00227, -0.00124]$ & $9.2 \times 10^{-90}$ \\
$0.15$ & $-0.00199$ & $[-0.00268, -0.00143]$ & $5.8 \times 10^{-76}$ \\
$0.20$ & $-0.00180$ & $[-0.00240, -0.00129]$ & $5.5 \times 10^{-57}$ \\
$0.30$ & $-0.00232$ & $[-0.00399, -0.00118]$ & $5.8 \times 10^{-5}$ \\
$0.40$ & $-0.02110$ & $[-0.02970, -0.01339]$ & $3.2 \times 10^{-9}$ \\
$0.50$ & $-0.32452$ & $[-0.33672, -0.31222]$ & $1.5 \times 10^{-95}$ \\
\bottomrule
\end{tabular}

\end{table}

\paragraph{Distribution shift: ImageNet-Sketch.} We evaluate the same DINOv2 CLS head on ImageNet-Sketch (3 seeds, cosine $\times 100$). At $\eta \in \{0.05, 0.10, 0.15, 0.20, 0.30, 0.40, 0.50\}$, importance gives $0.001 / 0.005 / 0.018 / 0.048 / 0.216 / 0.777 / 2.169$ and ToMe scoring gives $0.519 / 0.662 / 0.665 / 0.942 / 1.072 / 1.661 / 2.171$. The importance head keeps a more than $100\times$ advantage at low ratios. The gap closes at $\eta = 0.50$, but neither method is much worse than on in-distribution data.

\section{Additional Baselines and Classification Accuracy}
\label{sec:extra_baselines}
\paragraph{Frozen-backbone baselines.} Methods such as DynamicViT, A-ViT, and Token Cropr fine-tune the backbone together with their scoring mechanism, so they do not apply to a frozen backbone--decoder pair. We therefore compare against scoring rules that work on a frozen model, all inside the same CLS merge pipeline (\S\ref{sec:soft_merge}). \emph{EViT} scores tokens by the strength of the attention given by the CLS token \citep{evit}. \emph{Token norm} uses the $\ell_2$ norm of each token's features. \emph{CLS similarity} uses the cosine similarity of each token to the CLS token. Table~\ref{tab:extra_baselines} reports the results on DINOv2 CLS at probe layer 10. The importance head is best at every ratio, or tied with EViT at the three ratios where both round to the same two decimals. EViT is the strongest baseline and comes close at low ratios, but it needs the attention weights, whereas the head uses a single forward pass on the features. Both training-free baselines beat ToMe scoring. This table comes from a separate single-seed run, so its values differ from Table~\ref{tab:cls_pruning}, which averages three seeds. The differences are visible for the baselines, for example ToMe scoring reads $0.53$ here against $0.67$ there at $\eta = 0.05$.

\begin{table}[h]
\centering
\caption{Frozen-backbone baselines, DINOv2 CLS L10, ImageNet-val @ 224. CLS cosine-distance degradation ($\times 100$), lower is better. Single seed, rounded to two decimals, so entries that tie at low ratios are not separated here.}
\label{tab:extra_baselines}
\begin{tabular}{lccccccc}
\toprule
Method & $\eta{=}0.05$ & $0.10$ & $0.15$ & $0.20$ & $0.30$ & $0.40$ & $0.50$ \\
\midrule
Importance     & \textbf{0.00} & \textbf{0.01} & \textbf{0.02} & \textbf{0.06} & \textbf{0.24} & \textbf{0.82} & \textbf{2.17} \\
EViT           & \textbf{0.00} & \textbf{0.01} & 0.03 & \textbf{0.06} & 0.28 & 0.88 & 2.30 \\
Token norm     & 0.04 & 0.14 & 0.25 & 0.38 & 0.78 & 1.45 & 2.59 \\
CLS similarity & 0.02 & 0.08 & 0.17 & 0.42 & 0.90 & 1.70 & 2.91 \\
ToMe score     & 0.53 & 0.71 & 0.97 & 0.99 & 1.77 & 2.28 & 3.69 \\
Random         & 0.14 & 0.36 & 0.58 & 0.83 & 1.29 & 2.07 & 3.48 \\
\bottomrule
\end{tabular}
\end{table}

\paragraph{Classification accuracy.} To check that embedding fidelity translates into classification accuracy, we classify the pruned DINOv2 CLS embeddings on ImageNet-val with a weighted $k$-NN classifier ($k = 20$, temperature $0.07$, a bank of 20{,}000 training embeddings). The unpruned model reaches top-1 $0.674$. For every method and every prune ratio, top-1 stays between $0.66$ and $0.68$. At $\eta = 0.50$ it is $0.670$ for the importance head, $0.668$ for ToMe scoring, and $0.664$ for random selection. Pruning is therefore safe for classification in this setting. The differences between methods are too small to separate them, because the classifier only needs the CLS token, which is never pruned. Cosine distance to the unpruned embedding is more sensitive, which is why we use it as the main metric.

\section{Token Importance Visualizations}
\label{sec:token_importance_visualizations}
Figure~\ref{fig:depth_importance} shows per-token importance maps for the DPT depth task on four ImageNet-val images, comparing the importance head with ToMe cosine scoring. Each heatmap is the per-token score reshaped to the $16{\times}16$ patch grid and alpha-blended over the input. White outlines mark the top-$25\%$ tokens that pruning would keep. The bottom row shows the depth map predicted by the frozen DPT decoder, so that the reader can see what task the importance signal serves. Each heatmap is annotated with its per-image Spearman $\rho$ against the target $\mathrm{imp}^*$.

Two qualitative observations stand out. First, the importance head's map concentrates on depth discontinuities, such as the curtain on the dark window, the dog's body against the rug, the hairdryer body, and the poodle's outline, and it agrees well with the target ($\rho \in [0.87, 0.95]$). These four images are illustrative. Their per-image $\rho$ is well above the held-out average of $\approx 0.20$ (\S\ref{sec:structural}), so they show where the head does well rather than typical agreement. The pattern is what the $\sigma^2$-weighted top-$r$ projection of Equation~\ref{eq:imp_target} predicts. DPT is delocalized on average (CV $\approx 0.026$ at the L02 and L10 probes of Table~\ref{tab:regime_taxonomy}), but the locally task-salient directions cluster on object boundaries, and the importance head recovers them. Second, ToMe's cosine-similarity heatmaps are nearly uncorrelated with the target ($\rho \in [-0.14, 0.27]$). They highlight tokens that differ from their neighbours rather than tokens the depth decoder is sensitive to, and so they pick scattered patches in the background and in textured regions instead of the depth-relevant regions.

\begin{figure}[h]
\centering
\includegraphics[width=\textwidth]{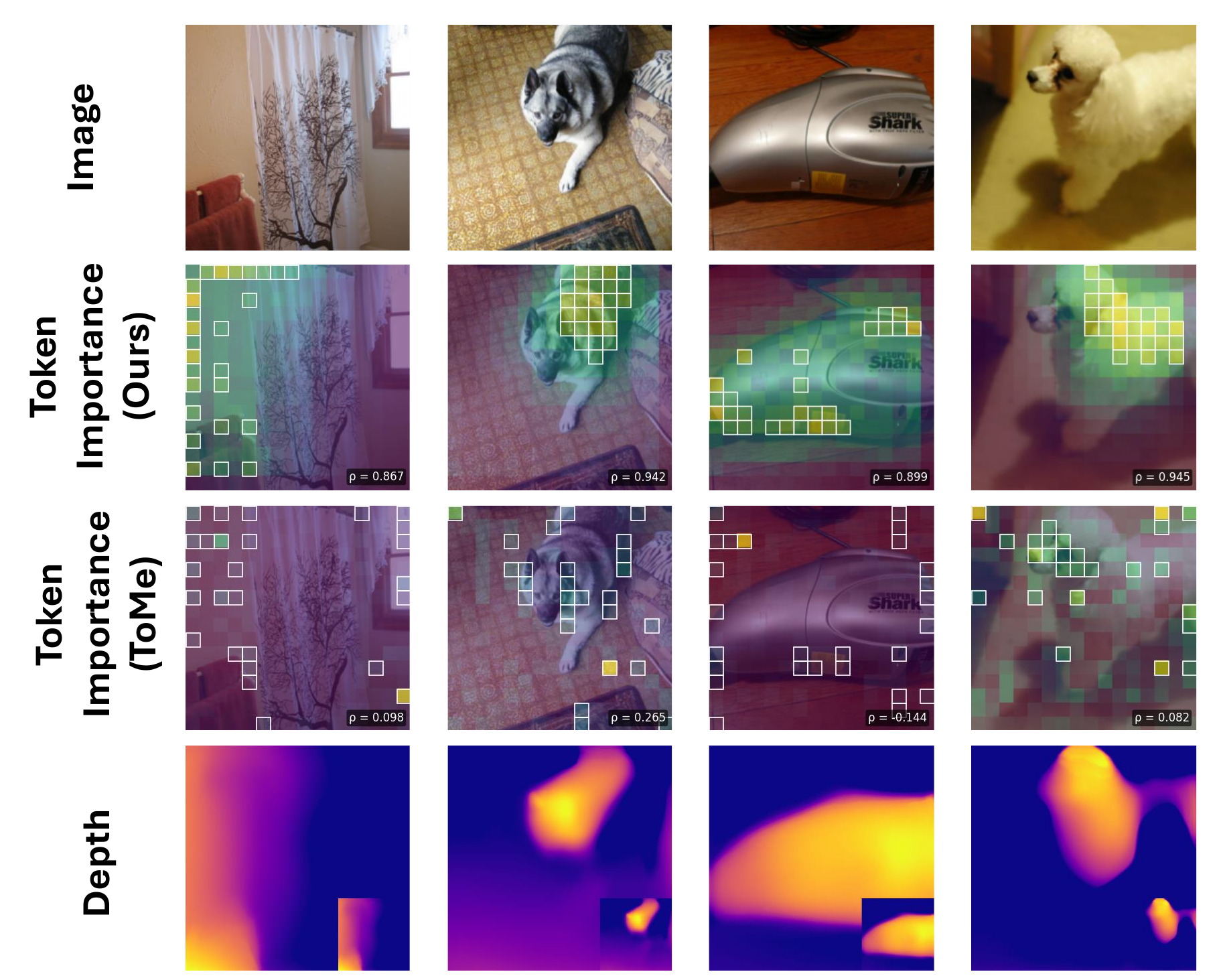}
\caption{DPT depth token importance on four ImageNet-val images. \emph{Row 1:} input. \emph{Row 2:} importance-head scores, with white outlines on the top-$25\%$ tokens and per-image Spearman $\rho$ against the target $\mathrm{imp}^*$ in the bottom-right corner. \emph{Row 3:} ToMe cosine-similarity scores, with the same overlay. \emph{Row 4:} depth predicted by the frozen DPT head. The importance head concentrates on depth discontinuities and matches the target at $\rho \in [0.87, 0.95]$ on these examples (the held-out average is $\approx 0.20$). ToMe is nearly uncorrelated with the target ($\rho \in [-0.14, 0.27]$).}
\label{fig:depth_importance}
\end{figure}

\section{SPN VAE Validation: Full Setup and Results}
\label{sec:spn_vae_validation}
We instantiate the VAE-compression pipeline on VGGT depth, which is in regime~(i) without compression (\S\ref{sec:vae_extension}). The probe point is the VAE latent $z \in \mathbb{R}^{4096}$ and the probe map is $z \mapsto \phi(G(z))$: the VAE decoder reconstructs the head-input features, which the frozen VGGT depth head uses. The SPN takes VGGT layer-23 features as input and predicts $U \in \mathbb{R}^{4096 \times r}$ in the latent space. Because this Jacobian has only $4096$ columns, an exact SVD gives the ground-truth subspace against which we measure $V_1$. We train the SPN via Algorithm~\ref{alg:randnla} and ablate the role of the RandNLA scheme by comparing three SPN variants with otherwise identical architecture and training data:
\begin{itemize}
    \item \textbf{Random init}, an SPN initialized at random with no supervision (worst-case baseline).
    \item \textbf{RandNLA $q = 1$}, our Algorithm~\ref{alg:randnla} with a single power-iteration step (cheapest meaningful setting).
    \item \textbf{Spectral}, supervision computed at every step from a high-rank deterministic SVD of $J$. This is the \emph{supervision ceiling} that RandNLA approximates, i.e.\ the best this SPN reaches when given exact targets.
\end{itemize}
Table~\ref{tab:spn_vae_variants} reports the subspace-alignment metric $V_1 = \tfrac{1}{r}\|U(F;\theta)^\top V_r^*\|_F^2 \in [0, 1]$ between the SPN's predicted modes and the true rank-$r$ task-sensitive subspace ($V_1 = 1$ is perfect alignment). Two conclusions follow.

\begin{table}[h]
\centering
\caption{SPN training on VAE-compressed VGGT depth. Subspace alignment $V_1$ between predicted modes and the true rank-$r$ subspace, higher is better. Random initialization has no signal, and RandNLA at the cheapest setting matches the supervision ceiling within noise. The difference between RandNLA and Spectral can be attributed to some noise and smoothing properties of RandNLA. }
\label{tab:spn_vae_variants}
\begin{tabular}{lc}
\toprule
SPN variant & $V_1$ \\
\midrule
Random init                              & $0.005$ \\
RandNLA, $q = 1$ (ours)                  & $\mathbf{0.713}$ \\
Spectral (exact SVD targets)          & $0.710$ \\
\bottomrule
\end{tabular}
\end{table}

First, the gap between random and structured supervision ($0.005$ vs.\ $0.713$, a $\approx 140\times$ improvement) shows that the SPN learns task geometry and not training-set artifacts. A random initialization does not come close to the true subspace, so the improvement comes from fitting the Jacobian's top-$r$ directions. Second, RandNLA at $q = 1$ matches the supervision ceiling within noise ($0.713$ vs.\ $0.710$). The ceiling sits at $\approx 0.71$ rather than $1$ because $V_1$ is an \emph{amortized} alignment. A single feature-conditioned network must predict the top-$r$ subspace for all images from features alone, and it cannot reproduce every image's SVD exactly. The $\approx 0.29$ shortfall is therefore the SPN's approximation error, not a limitation of the SVD or of RandNLA, which is why exact deterministic supervision (Spectral) does no better than cheap randomized supervision (RandNLA, $q = 1$). The saturation at $q = 1$ shows that, for the VAE-compressed Jacobian, a single power-iteration step already resolves the top-$r$ subspace. This is consistent with the geometric rate of Theorem~\ref{thm:convergence}, but it is an empirical finding, since the theorem's constant is too loose to predict it.

This experiment supports two claims. (a)~The VAE pathway of \S\ref{sec:vae_extension} works in practice. An intractable dense decoder becomes a compressed pipeline in which the SPN learns the rank-$r$ task-sensitive subspace, up to the amortization gap discussed above. (b)~The RandNLA training algorithm of \S\ref{sec:randnla_training} reaches the supervision ceiling at the cheapest setting, which supports both Algorithm~\ref{alg:randnla} and Theorem~\ref{thm:convergence}.

\section{Task-Induced Retrieval}
\label{sec:task_retrieval}
A second application of the learned geometry is task-sensitive image retrieval. Instead of scoring image similarity by Euclidean or cosine distance, we score it by a task-induced distance derived from the rank-$r$ pullback metric.

\paragraph{Setup.} Given a query feature $F_q$ and a gallery $\{F_i\}$ extracted from DINOv2 ViT-B/14 at $\ell = 10$, we rank the gallery by predicted distance to $F_q$. The reference distance is the depth-decoder distance $\|\phi(F_q) - \phi(F_i)\|$ between the dense outputs (AbsRel-normalised). Galleries are constructed from ImageNet-val (in-distribution) and from NYU-Depth-V2 (cross-dataset, training on ImageNet only). We report Spearman $\rho$ and Kendall $\tau$ between predicted and reference ranking, and MAE/RMSE between predicted and reference distance values, averaged over a held-out query set.

Both rank statistics compare two orderings of the same gallery, the one induced by the predicted distances and the one induced by the reference distances. \emph{Spearman $\rho$} replaces every distance by its rank in its own ordering and takes the Pearson correlation of those ranks, so it is $1$ whenever the predicted distances are any increasing function of the reference distances. \emph{Kendall $\tau$} instead counts pairs. For each pair of gallery items, the two orderings either agree (concordant) or disagree (discordant), and $\tau$ is the number of concordant pairs minus the number of discordant pairs, divided by the number of pairs. Both equal $1$ for identical orderings, $0$ for unrelated orderings, and $-1$ for exactly reversed orderings. Kendall $\tau$ is the stricter of the two, since a single swapped pair costs the same wherever it occurs, which is why its values are lower than $\rho$ throughout. Neither sees the scale of the distances, so we also report MAE and RMSE on the distance values themselves.

\paragraph{Distance definition.} The learned geometry is a field of seminorms on feature space. For a feature map $F$ and a perturbation $\Delta \in \mathbb{R}^{N \times D}$, write
\begin{equation}
\label{eq:drm_seminorm}
    \|\Delta\|_{g_\theta(F)} \;=\; s\,\big\| A(F)\,\Delta \big\|_2, \qquad
    A(F)\,\Delta \;:=\; W^\top \sum_{t=1}^{N} h_\theta(F)_t\, \Delta_t \;\in\; \mathbb{R}^{k},
\end{equation}
where $s \in \mathbb{R}_{>0}$ is a learned global scale, $W \in \mathbb{R}^{D \times k}$ is a learned low-rank factor, and $h_\theta(F)_t$ is the importance-head score at token $t$. This is the quadratic form of $g_\theta(F) = s^2 A(F)^\top A(F)$, which is positive semidefinite of rank at most $k$ and varies smoothly with $F$. The reading is that $\sum_t h_\theta(F)_t \Delta_t$ aggregates the per-token differences with task-dependent weights, and $W^\top$ projects that aggregate onto the rank-$k$ task-sensitive subspace.

A field of seminorms induces a length structure, and the associated distance is
\begin{equation}
\label{eq:drm_geodesic}
    d_{g_\theta}(F_1, F_2) \;=\; \inf_{\gamma} \int_0^1 \big\|\dot\gamma(t)\big\|_{g_\theta(\gamma(t))}\, dt,
\end{equation}
the infimum over piecewise-smooth paths from $F_1$ to $F_2$. This $d_{g_\theta}$ is a pseudometric. It is non-negative and vanishes on the diagonal, it is symmetric because reversing a path preserves its length, and it satisfies the triangle inequality because concatenating a path from $F_1$ to $F_2$ with one from $F_2$ to $F_3$ gives a path from $F_1$ to $F_3$. It is a pseudometric rather than a metric because $g_\theta$ is only positive semidefinite, so two distinct feature maps whose difference lies in $\ker A$ along the connecting path are at distance zero. This is the same degeneracy as in \S\ref{sec:background}, and it is intended: those directions are the ones the decoder ignores.

Computing (\ref{eq:drm_geodesic}) requires minimizing over paths at dimensionality $ND = 257 \times 768 \approx 2 \times 10^5$, which is prohibitive. We therefore evaluate the straight segment $\gamma(t) = F_1 + t\,\delta$ with $\delta = F_2 - F_1$, and approximate its length by the trapezoidal rule using the two endpoints,
\begin{equation}
\label{eq:drm_sym}
    d_{\text{sym}}(F_1, F_2) \;=\; \tfrac{1}{2}\Big[\, \|\delta\|_{g_\theta(F_1)} \;+\; \|\delta\|_{g_\theta(F_2)} \,\Big]
    \;=\; \tfrac{s}{2}\Big[\, \big\|A(F_1)\,\delta\big\|_2 \;+\; \big\|A(F_2)\,\delta\big\|_2 \,\Big].
\end{equation}
We call this the \emph{Distilled Riemannian Metric} (DRM), and it is what the reported numbers use. Each evaluation costs $\mathcal{O}(NDk)$ and needs only features, with no decoder pass and no matrix inversion.

Two caveats about (\ref{eq:drm_sym}) are worth stating plainly. First, the one-sided quantity $\|\delta\|_{g_\theta(F_1)}$ used inside it is not symmetric, since the seminorm is evaluated at one endpoint, which is why we symmetrize. Second, $d_{\text{sym}}$ is a quadrature of the length of one particular path rather than an infimum over paths, so it inherits non-negativity, symmetry and vanishing on the diagonal from (\ref{eq:drm_seminorm}), but it does not satisfy the triangle inequality in general. It is therefore a symmetric dissimilarity that approximates the pseudometric (\ref{eq:drm_geodesic}) to first order, and we use it as such. Where the triangle inequality is required, the corresponding object is $d_{g_\theta}$, or the constant-metric special case $\bar g = \mathbb{E}_F[g_\theta(F)]$, for which $\sqrt{\mathrm{vec}(\delta)^\top \bar g\, \mathrm{vec}(\delta)}$ is a genuine pseudometric because it comes from a fixed seminorm.

\paragraph{Training.} Triplets $(F_1, F_2, d_{\text{task}})$ are sampled by drawing image pairs from the training set, extracting layer-$\ell$ features, and computing $d_{\text{task}} = \|\phi(F_1) - \phi(F_2)\|$ from the frozen depth decoder. The loss combines a regression term on the task distance and a distillation term tying the per-token weights $h_\theta(F)_t$ to the Jacobian-derived target importance $\mathrm{imp}^*$ of \S\ref{sec:conformal_head}:
\[
\mathcal{L} \;=\; \mathcal{L}_{\text{reg}}\big(d_{\text{sym}}(F_1, F_2),\, d_{\text{task}}\big) \;+\; \lambda_d \, \mathcal{L}_{\text{distill}}.
\]
Without $\mathcal{L}_{\text{distill}}$, the weights $h_\theta$ have no anchor to the underlying geometry, and without $\mathcal{L}_{\text{reg}}$, the scale $s$ and projection $W$ are unsupervised. The two losses play different roles. $\mathcal{L}_{\text{reg}}$ calibrates the overall magnitude, and $\mathcal{L}_{\text{distill}}$ shapes the per-token contributions.

\paragraph{Methods compared.} Three feature-only distances on the same DINOv2 features. \emph{$L_2$} is $\|F_1 - F_2\|_F$ between layer-$\ell$ feature tensors. \emph{Cosine} uses CLS-pooled cosine similarity, the standard ViT retrieval baseline. \emph{DRM} is the symmetrised distance of Equation~\ref{eq:drm_sym}. All three operate at query time without re-running the depth decoder.

\paragraph{Results.} Table~\ref{tab:depth_retrieval_indist} reports the in-distribution evaluation. DRM's predicted ranking matches the reference at Spearman $\rho = 0.783$, against $0.223$ for $L_2$ and $0.206$ for cosine, a $3.5\times$ improvement in correlation. MAE drops from $11.43$ ($L_2$) and $0.677$ (cosine) to $0.233$. Table~\ref{tab:depth_retrieval_cross} reports the cross-dataset case (trained on ImageNet, evaluated on NYU-Depth-V2). DRM keeps its lead on every metric, with $\rho = 0.457$ against $0.338$ for cosine and MAE $0.246$ against $0.627$. The learned task-induced distance has the highest rank correlation in both settings. The gap narrows under distribution shift but does not collapse, which matches the pruning result that geometry trained on ImageNet transfers to NYU at lower but useful quality.

\begin{table}[h]
    \centering
    \caption{In-distribution depth retrieval (train ImageNet, test ImageNet, same-distribution split). Lower is better for MAE/RMSE; higher is better for Spearman/Kendall. Bold: best per column among learned/baseline rows.}
    \label{tab:depth_retrieval_indist}
    \begin{tabular}{lcccc}
    \toprule
    Method            & Spearman $\rho$ & Kendall $\tau$ & MAE     & RMSE    \\
    \midrule
    $L_2$ baseline    & 0.223           & 0.153          & 11.43   & 11.55   \\
    Cosine baseline   & 0.206           & 0.139          & 0.677   & 0.805   \\
    DRM               & \textbf{0.783}  & \textbf{0.581} & \textbf{0.233} & \textbf{0.289} \\
    \midrule
    Reference (decoder distance) & 1.000 & 1.000 & 0.000 & 0.000 \\
    \bottomrule
    \end{tabular}
\end{table}

\begin{table}[h]
    \centering
    \caption{Cross-dataset generalization: methods trained on ImageNet, evaluated on NYU-Depth-V2. The learned task-induced distance retains a clear lead over the Euclidean and cosine baselines on every metric.}
    \label{tab:depth_retrieval_cross}
    \begin{tabular}{lcccc}
    \toprule
    Method            & Spearman $\rho$ & Kendall $\tau$ & MAE     & RMSE    \\
    \midrule
    $L_2$ baseline    & 0.266           & 0.178          & 7.77    & 7.90    \\
    Cosine baseline   & 0.338           & 0.229          & 0.627   & 0.707   \\
    DRM               & \textbf{0.457}  & \textbf{0.312} & \textbf{0.246} & \textbf{0.310} \\
    \bottomrule
    \end{tabular}
\end{table}

\section{Glossary of Notation}
\label{sec:glossary}
\begin{center}
\small
\begin{tabular}{p{2.6cm}p{10.4cm}}
\toprule
Symbol / term & Meaning \\
\midrule
$F^\ell \in \mathbb{R}^{N \times D}$ & ViT features at probe layer $\ell$; $N$ tokens of dimension $D$. \\
$\phi$ & Frozen task decoder (DPT depth head, CLS embedding, VGGT heads). \\
$\psi_\ell$, $M$ & Probe map from layer-$\ell$ features to the $M$-dimensional task output. \\
$J$ & Jacobian of the probe map's output with respect to $F^\ell$, of size $M \times ND$. \\
$g = J^\top J$ & Pullback metric: $v^\top g v = \|Jv\|^2$ is the squared first-order output change under perturbation $v$. Positive semidefinite. \\
$\sigma_j$, $v_j$ & Singular values and right singular vectors of $J$, largest first. \\
$r$ & Number of directions kept in the low-rank metric ($r = 20$ by default). \\
$\kappa_{cap}(r)$ & Fraction of $\|J\|_F^2$ carried by the top-$r$ directions; the tractability diagnostic. \\
$r_{\text{eff}}$ & Entropic effective rank of the squared singular values (\emph{trunc}: top-$r$ only; \emph{SLQ}: full spectrum). \\
CV & Coefficient of variation of the per-token energy of the $v_j$ (Eq.~\ref{eq:cv}); low CV means delocalized. \\
$m$, $q$, $r_0$ & Hutchinson probes, power-iteration rounds, and sketch width. \\
$\mathrm{imp}^*$ & Jacobian-derived per-token importance target (Eq.~\ref{eq:imp_target}). \\
SPN, $g_\theta$ & Spectral Pullback Network; predicts the rank-$r$ metric from features (Eq.~\ref{eq:spn}). \\
Importance head, $h_\theta$ & Small feature-only network predicting $\mathrm{imp}^*$ per token. \\
LLF & Last-Layer Fusion; runs the final ViT block on the full token sequence to refresh hard-pruned tokens. \\
$\eta$ & Prune ratio: the fraction of tokens removed. \\
$V_1$ & Subspace alignment $\tfrac{1}{r}\|U^\top V_r^*\|_F^2$ between predicted and true top-$r$ subspaces. \\
\bottomrule
\end{tabular}
\end{center}

\section{Limitations}
\label{sec:limitations}
\paragraph{Scope of the measured geometry.} We measure $\kappa_{cap}$ at probe layers that feed task decoders, and not at internal attention or MLP outputs. Whether those internal layers fall into the same regimes is open, and answering it would need a different definition of the probe map, but should be doable with the same techniques.

Additionally, the compressed pullback taken through the VAE-reconstructed features rather than the original ones approximates $g$ well only in the directions the VAE reconstructs well. We have no analytic bound on the discrepancy, and the empirical evidence for the pathway is the drop in effective rank plus the SPN result of Appendix~\ref{sec:spn_vae_validation}.

\paragraph{Transfer of the learned models.} The importance head is trained for one (decoder, probe layer) pair. We have not tested whether a head transfers to a different decoder, so each new task currently needs its own head. The diagnostic itself has no such restriction, since it is computed directly from the Jacobian of whatever pair is of interest.

\paragraph{When geometric pruning helps.} The advantage over ToMe scoring depends on the configuration and on the dataset, discussed in \S\ref{sec:pruning_results}. Three cases are worth repeating here. Single-stage pruning loses at low prune ratios on ImageNet, and at $\eta = 0.50$ on NYU-Depth-V2. Four-stage pruning fails at high ratios, where repeated pruning compounds the residuals. Single-stage pruning is also sensitive to input resolution and loses to ToMe scoring at $448 \times 448$ on ImageNet. The two-stage schedule is the one configuration that wins at every ratio and resolution we tested.

\paragraph{Cost at inference.} The method targets output quality under pruning, not throughput. At $224 \times 224$ the pruned model is slightly slower in wall-clock time than the unpruned one, despite removing $12.1\%$ of the backbone's FLOPs, because the fixed cost of scoring and re-inserting tokens exceeds the attention savings at 257 tokens. A speedup appears only at higher resolution, and even there it is modest: $1.12\times$ at $448 \times 448$, against $1.15\times$ for ToMe scoring.

\paragraph{Scope of the evaluation.} The evaluation is diagnostic rather than a systems study. Experiments use a single A40 and between $50$ and $2{,}000$ images, with ToMe scoring as the primary dense-pruning baseline. The additional frozen-backbone baselines of Appendix~\ref{sec:extra_baselines} cover DINOv2 CLS only, so we have no comparable baseline sweep on dense decoders.

\section{Discussion}
\label{sec:discussion}
A common assumption in deep representation learning is that better features lead to better downstream performance. Our results suggest a complementary view. When the features are already as expressive as those of a ViT-B/14 backbone, what matters most is which directions in feature space the downstream task actually uses, and these directions can be recovered from the decoder's Jacobian without retraining the backbone. In this sense, learning which directions matter can be as valuable as learning better features. We believe that this idea can be used to target a whole range of improved methods, not just restricted to tasks such as token pruning and merging, and not just restricted to ViTs. 

\section{Future Work}
\label{sec:future_work}
Several follow-ups build on the importance head and the metric it distills. Because the head supplies a feature-only metric without backpropagating through the decoder, it enables \emph{Riemannian flow matching} under task-induced metrics \citep{flow_matching, riemannian_flow_matching}, which raises the question of whether a task-sensitive velocity field produces straighter trajectories than Euclidean flow matching. It also enables \emph{geodesic image interpolation and editing} as a metric-aware analogue of linear feature interpolation \citep{latent_space_oddity, geometry_gans}. 

Computing $\kappa_{cap}$ per token would give a matrix-free local-reliability estimate for dense prediction without retraining. Two further directions use the diagnostic itself. Applying $\kappa_{cap}$ to LLM feature spaces, with the probe map given by the unembedding composed with the residual stream, could indicate whether attribution methods such as integrated gradients \citep{integrated_gradients} or SmoothGrad \citep{smoothgrad} are reliable for a given decoder, at a cost of $\mathcal{O}(m + rq)$ Jacobian-vector products per layer. Using $\mathrm{tr}(g(F))$ together with the $\kappa_{cap}$ defect as out-of-distribution scores would give matrix-free OOD detection that can be benchmarked against energy-based \citep{energy_ood} and Mahalanobis \citep{mahalanobis_ood} baselines. Other directions are cross-task metric transfer, optimal-transport retrieval with $d_g$ as the ground cost \citep{sinkhorn}, and estimating the sectional curvature of $(F, g)$ \citep{geometry_hvae, understanding_diffusion_models}.

\end{document}